\documentclass[11pt]{article}
\newif\ifstoc
\stoctrue 
\stocfalse
\pdfoutput=1
\usepackage{fullpage}
\usepackage{times}
\usepackage{graphicx,color}
\usepackage{array,float}
\usepackage{url}
\usepackage[usenames,dvipsnames]{xcolor}
\usepackage{amstext,amssymb,amsmath}
\usepackage{amsthm}
\usepackage{mathtools}
\usepackage{verbatim}
\usepackage{bm}
\usepackage{paralist}
\usepackage{ulem}
\usepackage[numbers]{natbib}
\usepackage{todonotes}
\usepackage[noend]{algorithmic}
\usepackage{algorithm}

\usepackage{geometry}

\ifstoc

\else

\fi

\newcommand{\stocoption}[2]{{\ifstoc #1 \else #2 \fi}}
\ifstoc
\else
\fi

\newcommand{\nuc}[1]{\|#1\|_{\sf nuc}}
\newcommand{\Mtheta}{\bm{\theta}}

\newcommand{\frob}[1]{\|#1\|_{F}}
\newcommand{\mink}[1]{\|#1\|_{\C}}
\newcommand{\minkQ}[1]{\|#1\|_{\Q}}
\newcommand{\minkQD}[1]{\|#1\|_{\Q^*}}
\newcommand{\minkq}[1]{\|#1\|_{\Q,q}}
\newcommand{\mixed}[1]{\|#1\|_{(k,\ell_{1,2})}}
\newcommand{\minkD}[1]{\|#1\|_{\C^*}}
\newtheorem{lem}{Lemma}[section]

\newcommand{\linfty}[1]{\|#1\|_\infty}
\newtheorem{thm}[lem]{Theorem}
\newtheorem{cor}[lem]{Corollary}

\newtheorem{defn}[lem]{Definition}
\newtheorem{fact}[lem]{Fact}
\newcommand{\poly}{\text{poly}\,}
\newcommand{\polylog}{\text{polylog}\,}

\newtheorem{claim}[lem]{Claim}
\newcommand{\ip}[2]{\langle #1,#2\rangle}

\newcommand{\nptheta}{\hat\theta}

\newcommand{\privtheta}{\theta^{priv}}
\renewcommand{\paragraph}[1]{\vspace{3pt}\noindent\textbf{#1}}

\newcommand{\ltwo}[1]{\|#1\|_2}
\newcommand{\lone}[1]{\|#1\|_1}

\newcommand{\eps}{\epsilon}
\newcommand{\A}{\mathcal{A}}
\newcommand{\D}{\mathcal{D}}
\newcommand{\G}{\mathcal{G}}

\newcommand{\I}{\mathbb{I}}

\newcommand{\E}{\mathbb{E}}

\newcommand{\F}{\mathcal{C}}
\newcommand{\empL}{\mathcal{L}}

\newcommand{\htheta}{\widetilde\theta}

\newcommand{\re}{\Re}

\newcommand{\Q}{\mathcal{Q}}

\newcommand{\grad}{\bigtriangledown}

\newcommand{\mypar}[1]{\smallskip
\noindent{\bf\em {#1}:}}

\newcommand{\ktnote}[1]{}

\newcommand{\ignore}[1]{}
\newcommand{\Jpriv}{{{J}^{\text{priv}}}}

\renewcommand{\b}{b}

\newcommand{\C}{\mathcal{C}}
\newcommand{\bregDiv}[1]{\mathcal{B}_{#1}}

\newcommand{\sgn}{{\sf sign}}
\DeclareMathOperator*{\argmin}{{\sf argmin}}

\theoremstyle{remark}
\newtheorem{remark}{Remark}
\def\pfw{\A_{\sf Noise-FW}}
\def\pfwe{\A_{\sf Noise-FW (polytope)}}
\def\pfwp{\A_{\sf Noise-FW (Gen-convex)}}
\def\sgn{\operatorname{sign}}
\begin{document}

\title{Private Empirical Risk Minimization Beyond the Worst Case:\\ The Effect of the Constraint Set Geometry\stocoption{\thanks{All the proofs and missing details appear in the full-version of the paper, attached as a continuation.}}{}}

\author{Kunal Talwar\thanks{Google. \texttt{kunal@kunaltalwar.org}.  (Part of this research was performed at the now defunct Microsoft Research Silicon Valley.)}
\and Abhradeep Thakurta\thanks{Yahoo Labs, Sunnyvale. \texttt{abhradeep@yahoo-inc.com}}
\and Li Zhang\thanks{Google. \texttt{liqzhang@google.com}.  (Part of this research was performed at the now defunct Microsoft Research Silicon Valley.)}
}

\maketitle
\thispagestyle{empty}

\begin{abstract}
Empirical Risk Minimization (ERM) is a standard technique in machine learning, where a model is selected by minimizing a loss function over constraint set. When the training dataset consists of private information, it is natural to use a differentially private ERM algorithm, and this problem has been the subject of a long line of work~\citep{CM08,KST12,JKT12,ST13sparse,DuchiJW13,JT14,BassilyST14,Ullman14}. A private ERM algorithm outputs an approximate minimizer of the loss function and its error can be measured as the difference from the optimal value of the loss function. When the constraint set is arbitrary, the required error bounds are fairly well understood~\citep{BassilyST14}. In this work, we show that the geometric properties of the constraint set can be used to derive significantly better results. Specifically, we show that a differentially private version of Mirror Descent leads to error bounds of the form $\tilde{O}(G_{\C}/n)$ for a Lipschitz loss function, improving on the $\tilde{O}(\sqrt{p}/n)$ bounds in~\cite{BassilyST14}. Here $p$ is the dimensionality of the problem, $n$ is the number of data points in the training set, and $G_{\C}$ denotes the Gaussian width of the constraint set that we optimize over. We show similar improvements for strongly convex functions, and for smooth functions. In addition, we show that when the loss function is Lipschitz with respect to the $\ell_1$ norm and $\C$ is $\ell_1$-bounded, a differentially private version of the Frank-Wolfe algorithm gives error bounds of the form $\tilde{O}(n^{-2/3})$. This captures the important and common case of sparse linear regression (LASSO), when the data $x_i$ satisfies $|x_i|_{\infty} \leq 1$ and we optimize over the $\ell_1$ ball. We also show our algorithm is nearly optimal by proving a matching lower bound for this setting.
\end{abstract}


\newpage
\pagenumbering{arabic}

\section{Introduction}

A common task in supervised learning is to select the model that best
fits the data. This is frequently achieved by selecting a {\em loss
function} that associates a real-valued loss with each datapoint $d$
and model $\theta$ and then selecting from a class of admissible
models, the model $\theta$ that minimizes the average loss over all
data points in the training set. This procedure is commonly referred
to as {\em Empirical Risk Minimization}(ERM).

The availability of large datasets containing sensitive information
from individuals has motivated the study of learning algorithms that
guarantee the privacy of individuals contributing to the database. A
rigorous and by-now standard privacy guarantee is via the notion of
differential privacy.  In this work, we study the design of
differentially private algorithms for Empirical Risk Minimization,
continuing a long line of
work initiated by~\cite{CM08} and continued in~\cite{CMS11,KST12,JKT12,ST13sparse,DuchiJW13,JT14,BassilyST14,Ullman14}.

As an example, suppose that the
training dataset $D$ consists of $n$ pairs of data $d_i=(x_i,y_i)$ where
$x_i\in\Re^p$, usually called the feature vector, and $y_i\in\Re$, the
prediction.  The goal is to find a ``reasonable model''
$\theta\in\Re^p$ such that $y_i$ can be predicted from the model $\theta$ and the feature vector $x_i$.  The quality of approximation is usually
measured by a loss function $\empL(\theta; d_i)$, and the empirical
loss is defined as $\empL(\theta;D)=\frac{1}{n}\sum_i\empL(\theta;
d_i)$.  For example, in the linear model with squared loss,
$\empL(\theta;d_i)=(\langle \theta, x_i \rangle- y_i)^2$. 
Commonly, one restricts $\theta$ to come from a constraint set $\C$. This can account for additional knowledge about $\theta$, or can be helpful in avoiding overfitting and making the learning algorithm more stable.
This leads to the constrained optimization problem of computing
$\theta^\ast = \argmin_{\theta\in\C}\empL(\theta;D)$.  For example, in the classical
sparse linear regression problem, we set $\C$ to be the $\ell_1$ ball. Now
our goal is to compute a model $\theta$ that is private with respect to changes in a single $d_i$ while having high quality, where the quality is measured by the excess empirical risk compared to the optimal model. 

\mypar{Problem definition} Given a data set $D=\{d_1,\cdots,d_n\}$ of
 $n$ samples from a domain $\D$, a convex set $\C\subseteq \re^p$, and
 a convex loss function $\empL:\C\times\D\to\re$, for any model
 $\theta$, define its excess empirical risk as
\begin{equation}
R(\theta;D) \stackrel{def}{=}
\frac{1}{n}\sum\limits_{i=1}^n\empL(\theta;d_i)-\min\limits_{\theta\in\C}\frac{1}{n}\sum\limits_{i=1}^n\empL(\theta;d_i).
\label{eq:empRisk}
\end{equation}
We define the {\em risk} of a mechanism $\A$ on a data set $D$ as
$R(\A;D)= \E[R(\A(D);D)]$, where the expectation is over the internal
randomness of $\A$, and the risk $R(\A) = \max_{D\in\D^n} R(\A;D)$ is the maximum risk over all the possible data sets. Our objective is then to
design a mechanism $\A$ which preserves
$(\epsilon,\delta)$-differential privacy (Definition
\ref{def:diffPrivacy}) and achieves as low risk as possible. We call the minimum achievable risk as {\em privacy risk}, defined as $\min_\A R(\A)$, where the min is over all $(\epsilon,\delta)$-differentially private mechanisms $\A$.



Previous work on private ERM has studied this problem under fairly
general conditions. For convex loss functions $\empL(\theta;d_i)$ that
for every $d_i$ are 1-Lipschitz as functions from $(\Re^p, \ell_2)$ to
$\Re$ (i.e. are Lipschitz in the first parameter with respect to the
$\ell_2$ norm), and for $\C$ contained in the unit $\ell_2$ ball,
~\cite{BassilyST14}
showed\footnote{Throughout the paper, we use
$\widetilde{O},\widetilde{\Omega}$ to hide the polynomial factors in
$1/\eps$, $\log(1/\delta)$, $\log n$, and $\log p$.} that the privacy
risk is at most $\widetilde{O}(\sqrt{p}/n)$. They also showed that
this bound cannot be improved in general, even for the squared loss
function. Similarly they gave tight bounds under stronger assumptions
on the loss functions (more details below).

In this work, we go beyond these worst-case bounds by exploiting
properties of the constraint set $\C$. In the setting of the previous
paragraph, we show that the $\sqrt{p}$ term in the privacy risk can be
replaced by the {\em Gaussian Width} of $\C$, defined as $G_{\C} ={
\E_{g \in \mathcal{N}(0,1)^p} [\sup\limits_{\theta \in \C} \langle \theta, g
\rangle]}$. Gaussian width is a well-studied quantity in Convex
Geometry that captures the global geometry of $\C$ \citep{ball1997elementary}.  For a $\C$
contained in the the $\ell_2$ ball it is never larger than $O(\sqrt{p})$
and can be significantly smaller. For example, for the $\ell_1$ ball,
the Gaussian width is only $\Theta(\sqrt{\log p})$. Similarly, we give improved bounds for other assumptions
on the loss functions. These bounds are proved by analyzing a noisy
version of the mirror descent algorithm \citep{NY83,beck2003mirror}. 

In the simplest setting, when the loss function
$\empL(\cdot, d)$ is convex, and $L_2$-Lipschitz with respect to the
$\ell_2$ norm on the parameter space, we get the following result. The precise bounds require a potential function that is tailored to the convex set
$\C$. In the following, let $\ltwo{\C}$ denote the $\ell_2$ radius of
$\C$, and $G_{\C}$ denote the Gaussian width of $\C$.
\begin{thm}[Informal version]
 There exists an $(\epsilon,\delta)$-differentially private algorithm
 $\A$ such that
 $$R(\A)=O\left(\frac{L_2G_{\C}\log(n/\delta)}{\epsilon
 n}\right).$$ In particular,
 $R(\A)=O\left(\frac{L_2\ltwo{\C}\sqrt{p}\log(n/\delta)}{\epsilon
 n}\right)$, and if $\C$ is a polytope with $k$ vertices,
 $R(\A)=O\left(\frac{L_2\ltwo{\C}\log k\log(n/\delta)}{\epsilon
 n}\right)$.
\end{thm}

Similar improvements can be shown
(Section~\ref{sec:corprivMirrorDesc}) for other constraint sets, such
as those bounding the grouped $\ell_1$ norm, interpolation norms, or the
nuclear norm when the vector is viewed as a matrix. When one
additionally assumes that the loss functions satisfy a strong
convexity definition \stocoption{(see full version)}{(Appendix~\ref{sec:tighterMirrorDesc})}, we can
get further improved bounds. Moreover, for smooth loss functions
(Section~\ref{sec:smoothStrong}), we can show that a simpler
objective perturbation algorithm \citep{CMS11,KST12} gives
Gaussian-width dependent bounds similar to the ones above.
Our work also implies Gaussian-width-dependent convergence bounds for the noisy (stochastic) mirror descent algorithm, which may be of independent interest.

The bounds based on mirror descent have a dependence on the $\ell_2$
Lipschitz constant. This constant might be too large for some
problems. For example, for the popular sparse linear regression
problem, one often assumes $x_i$ to have bounded $\ell_{\infty}$ norm,
i.e. each entry of $x_i$, instead of $\|x_i\|_2$, is bounded.  The
$\ell_2$ Lipschitz constant is then polynomial in $p$ and leads to a
loose bound. In these cases, it would be more beneficial to have a
dependence on the $\ell_1$ Lipschitz constant. Our next contribution is to address this issue. 
We show that when $\C$ is the $\ell_1$ ball, one can get
significantly better bounds using a differentially private version of
the Frank-Wolfe algorithm. Let $\lone{\C}$ denote the maximum $\ell_1$
radius of $\C$, and $\Gamma_\empL$ the curvature constant for $\empL$
(precise definition in Section~\ref{sec:frankWolfe}).
\begin{thm}
If $\C$ is a polytope with $k$ vertices, then there exists an
$(\epsilon,\delta)$-differentially private algorithm $\A$ such that
$$R(\A)=O\left(\frac{{{{\Gamma_\empL}}}^{1/3}\left(L_1\lone{\C}\right)^{2/3}\log(nk)\sqrt{\log(1/\delta)}}{(n\epsilon)^{2/3}}\right)\,.$$
In particular, for the sparse linear regression problem where each
$\|x_i\|_\infty\leq 1$, we have that
$$R(\A)=O(\log(np/\delta)/(n\epsilon)^{2/3})\,.$$
\end{thm}

Finally, we use the fingerprinting code lower bound technique developed
in~\cite{BUV13} to show that the upper bound for the sparse linear regression problem, and hence the above result, is nearly tight. 
\begin{thm}
For the sparse linear regression problem where $\|x_i\|_\infty\leq 1$,
for $\epsilon=0.1$ and $\delta=1/n$, any
$(\epsilon,\delta)$-differentially private algorithm $\A$ must have
$$R(\A)=\Omega(1/(n\log n)^{2/3})\,.$$
\end{thm}

In Table \ref{tab:bounds1} we summarize our upper and lower bounds. Combining our results with that of \cite{BassilyST14}, in particular we show that all the bounds in this paper are essentially tight. The lower bound for the $\ell_1$-norm case does not follow from \cite{BassilyST14}, and we provide a new lower bound argument.
 \begin{table}[tb]
 	\hspace{-0.5in}
 	\begin{tabular}{|p{1.22in}||p{1.1in}|p{1.1in}|p{2.4in}|c|}
 		\hline
 		& \multicolumn{2}{c|}{Previous work}  & \multicolumn{2}{c|}{\color{blue} This work} \\
 		\hline
 		Assumption & Upper bound & Lower bound & {\color{blue} Upper bound} & \color{blue} Lower bound\\
 		\hline
 		\hline
 		$1$-Lipschitz w.r.t $L_2$-norm and $\ltwo{\C}=1$ & $\frac{\sqrt p}{\epsilon n}$ \cite{BassilyST14} & $\Omega\left(\frac{\sqrt p}{n}\right)$ \cite{BassilyST14} 
 		& \color{blue} {\bf Mirror descent:} $\frac{1}{\epsilon n}{\scriptstyle \min\left\{\sqrt p,\log k \right\}}$ 
 		&   \\
 		\hline
 		... and $\lambda$-smooth
 		& $\frac{\sqrt p+\lambda}{\epsilon n}$ \cite{CMS11}
 		&  $\Omega\left(\frac{\sqrt p}{n}\right)$ \cite{BassilyST14}\newline (for $\lambda=O(p)$) &  \color{blue} {\bf Frank-Wolfe:} $\frac{\lambda^{1/3}}{(\epsilon n)^{2/3}}{\scriptstyle \min\left\{p^{1/3},\log^{1/3}k\right\}}$ 
 		&   \\
 		& & & 
 		\color{blue} {\bf Obj. pert:} $\frac{\min\left \{\sqrt{p},\sqrt{\log k}\right \}+\lambda}{\epsilon n}$ & \\
 		\hline 
 		\hline
 		$1$-Lipschitz w.r.t $L_1$-norm, $\lone{\C}=1$,
 		and curvature $\Gamma$ & & 
 		& \color{blue} {\bf Frank-Wolfe:} $\frac{\Gamma^{1/3}\log(nk)}{(\epsilon n)^{2/3}}$
 		& \color{blue} $\tilde\Omega\left(\frac{1}{n^{2/3}}\right)$\\
 		\hline
 		\hline
 	\end{tabular}
 	\caption{Upper and lower bounds for $(\epsilon,\delta)$-differentially private ERM. $k$ denotes the number of corners in the convex set $\C$.(In general the dependence is on the Gaussian width of $\C$, generalizing $\sqrt p$ or $\sqrt{\log k}$.) The curvature parameter is a weaker condition than smoothness, and is in particular bounded by the smoothness. Bounds ignore multiplicative dependence of $\log(1/\delta)$ and in the lower bounds, $\epsilon$ is considered as a constant. The lower bounds of \cite{BassilyST14} have the form $\Omega(\min\{n,\cdots\})$.}
 	\label{tab:bounds1}
 \end{table}

Our results enlarge the set of problems for which privacy comes ``for
free''. Given $n$ samples from a distribution, suppose that
$\theta^{\ast}$ is the empirical risk minimizer and $\privtheta$ is
the differentially private approximate minimizer. Then the
non-private ERM algorithm outputs $\theta^\ast$ and incurs expected
(on the distribution) loss equal to the ${\sf loss}(\theta^*,{\mbox{training-set}}) + {\mbox{generalization-error}}$, where the  {\em generalization error} term depends on
the loss function, $\C$ and on the number of samples $n$.  The
differentially private algorithm incurs an additional loss of the
privacy risk. If the privacy risk is asymptotically no larger than the
generalization error, we can think of privacy as coming for free,
since under the assumption of $n$ being large enough to make the
generalization error small, we are also making $n$ large enough to
make the privacy risk small. For many of the problems, by our work we
get privacy risk bounds that are close to the best known
generalization bounds for those settings. More concretely, in the case when the $\ltwo{\C}\leq 1$ and the loss function is $1$-Lipschitz in the $\ell_2$-norm, the known generalization error bounds strictly dominate the privacy risk when $n=\omega(G_\C^4)$ \cite[Theorem 7]{SSSS}. In the case when $\C$ is the $\ell_1$-ball, and the loss function is the squared loss with $\linfty{x}\leq 1$ and $|y|\leq 1$, the generalization error dominates the privacy risk when $n=\omega(\log^3 p)$ \cite[Theorem 18]{bartlett2003rademacher}.

\subsection{Related work}
 
In the following we distinguish
between the two settings: i) the convex set is bounded in the
$\ell_2$-norm and the the loss function is $1$-Lipschitz in the
$\ell_2$-norm (call it the $(\ell_2/\ell_2)$-setting for brevity), and
ii) the convex set is bounded in the $\ell_1$-norm and the the loss
function is $1$-Lipschitz in the $\ell_1$-norm (call it the
$(\ell_1/\ell_\infty)$-setting).

\mypar{The $(\ell_2/\ell_2)$-setting} In all the works on private convex optimization that we are aware of, either the excess risk guarantees depend polynomially on the dimensionality of the problem ($p$), or assumes special structure to the loss (e.g., generalized linear model \cite{JT14} or linear losses \cite{DNPR10,ST15}). Similar dependence is also present in the online version of the problem \citep{JKT12,ST13online}. \cite{BassilyST14} recently show that in the private ERM setting, in general this polynomial dependence on $p$ is unavoidable. In our work we show that one can replace this dependence on $p$ with the Gaussian width of the constraint set $\C$, which can be much smaller. We use the mirror descent algorithm of ~\cite{beck2003mirror} as our building block.
 
\mypar{The $(\ell_1/\ell_\infty)$-setting} The only  results in this setting that we are aware of are \cite{KST12,ST13sparse,JT14,ST15}.
The first two works make certain assumtions about the instance (\emph{restricted strong convexity} (RSC) and
\emph{mutual incoherence}). Under these assumptions, they obtain privacy risk guarantees that depend
logarithmically in the dimensions $p$, and thus allowing the
guarantees to be meaningful even when $p\gg n$. In fact their bound of $O(\polylog p / n)$ can be better than our {\em tight} bound of $O(\polylog p / n^{2/3})$. However, these
assumptions on the data are strong and may not hold in practice
\citep{larry}. Our guarantees do not require
any such data dependent assumptions. 
The result of \cite{JT14} captures the scenario when the constraint set $\C$ is the probability simplex and the loss function is in the generalized linear model, but provides a \emph{worse} bound of   
$O(\polylog p / n^{1/3})$. 

\mypar{Effect of Gaussian width in risk minimization} For linear losses, the notions of Rademacher complexities and Gaussian complexities are closely related to the concept of Gaussian width, i.e., when the loss function are of the form $\ip{\theta}{d}$. One of the initial works that formalized this connection was by \cite{bartlett2003rademacher}. They in particular bound the excess generalization error by the Gaussian complexity of the constraint set $\C$, which is very similar to Gaussian width in the context of linear functions. Recently \cite{chandrasekaran2012convex} show that the Gaussian width of a constraint set $\C$ is very closely related to the number of generic linear measurements one needs to perform to recover an underlying model $\theta^\ast\in\C$. 

\cite{shamir2013stochastic} analyzed the problem of noisy stochastic gradient descent (SGD) for general convex loss functions. Their empirical risk guarantees depend polynomially on the $\ell_2$-norm of the noise vector that gets added during the gradient computation in the SGD algorithm. As a corollary of our results we show that if the noise vector is sub-Gaussian (not necessarily \emph{spherical}), the polynomial dependence on the $\ell_2$-norm of the noise can be replaced by a term depending on the Gaussian width of the set $\C$.

\mypar{Analysis of noisy descent methods} The analysis of noisy versions of gradient descent and mirror descent algorithms has attracted interest for unrelated reasons~\citep{recht2011hogwild,duchi2013estimation} when asynchronous updates are the source of noise. To our knowledge, this line of work does not take the geometry of the constraint set into account, and thus our results may be applicable to those settings as well.

\stocoption{}{We should note here that the notion of Gaussian width has been used
by \cite{NTZ13}, and \cite{dwork2013efficient} in the context of differentially
private query release mechanisms but in the very different context of answering multiple linear queries over a database.}

\section{Background}
\label{sec:background}
\subsection{Differential Privacy}
\label{sec:privacy}

The notion of differential privacy (Definition \ref{def:diffPrivacy}) is by now a defacto standard for statistical data privacy \citep{DMNS06,Dwork06,Dwork08,Dwork09}. One of the reasons for which differential privacy has become so popular is because it provides meaningful guarantees even in the presence of arbitrary auxiliary information.  At a semantic level, the privacy guarantee ensures that \emph{an adversary learns almost the same thing about an individual independent of his presence or absence in the data set.} The parameters $(\epsilon,\delta)$ quantify the amount of information leakage. For reasons beyond the scope of this work, $\epsilon\approx 0.1$ and $\delta=1/n^{\omega(1)}$ are a good choice of parameters. Here $n$ refers to the number of samples in the data set. 

\begin{defn}
	A randomized  algorithm $\A$ is $(\eps,\delta)$-differentially private (\cite{DMNS06,DKMMN06}) if, for all neighboring data sets $\D$ and $\D'$ (i.e., they differ in one record, or equivalently, $d_H(D,D') = 1$) and for all events $S$ in the output space of $\A$, we have
 $$\Pr(\A(\D)\in S) \leq e^{\eps} \Pr(\A(\D')\in S) +\delta\,.$$
 Here $d_H(D,D')$ refers to the Hamming distance.
	 \label{def:diffPrivacy}
\end{defn}
\stocoption{\subsection{Bregman Divergence, Lipschitz Continuity and Strong Convexity}
	\label{sec:convNormGauss}}{
\subsection{Bregman Divergence, Convexity, Norms, and Gaussian Width}
\label{sec:convNormGauss}
}
In this section we review some of the concepts commonly used in convex optimization useful to the exposition of our algorithms. In all the definitions below we assume that the set $\C\subseteq\re^p$ is closed and convex.
\stocoption{}{
	
\mypar{$\ell_q$-norm, $q\geq 1$} For $q \geq 1$, the $\ell_q$-norm for any vector $v\in\re^p$ is defined as $\left(\sum\limits_{i=1}^p v(i)^q\right)^{1/q}$, where $v(i)$ is the $i$-th coordinate of the vector $v$.
}

\mypar{Minkowski norm w.r.t a set $\C\subseteq\re^p$} For any vector $v\in\re^p$, the Minkowski norm (denoted by $\mink{v}$) w.r.t. a centrally symmetric convex set $\C$ is defined as follows.
$$\mink{v}=\min\limits\{r\in\re:v\in r\C\}.$$ 

\mypar{$L$-Lipschitz continuity w.r.t. norm $\|\cdot\|$} A function $\Psi:\C\to\re$ is $L$-Lispchitz within a set $\C$ w.r.t. a norm $\|\cdot\|$ if the following holds.

$$\forall\theta_1,\theta_2\in\C,\left|\Psi(\theta_1)-\Psi(\theta_2)\right|\leq L\cdot \|\theta_1-\theta_2\|.$$

\stocoption{\mypar{$\Delta$-strong convexity w.r.t norm
		$\|\cdot\|$} A function is $\Delta$-strongly convex within a set $\C$ w.r.t. a norm $\|\cdot\|$ if
	$$\forall\theta_1,\theta_2\in\C, \alpha\in[0,1],\Psi(\alpha\theta_1+(1-\alpha)\theta_2)\leq\alpha\Psi(\theta_1)+(1-\alpha)\Psi(\theta_2)-\frac{\Delta\cdot\alpha(1-\alpha)}{2}\|\theta_1-\theta_2\|^2.$$}{

\mypar{Convexity and $\Delta$-strong convexity w.r.t norm
  $\|\cdot\|$} A function $\Psi:\C\to\re$ is convex if
$$\forall \theta_1,\theta_2 \in \C, \alpha \in [0,1], \Psi(\alpha
\theta_1+ (1-\alpha)\theta_2) \leq \alpha \Psi(\theta_1) + (1-\alpha)\Psi(\theta_2).$$
A function is $\Delta$-strongly convex within a set $\C$ w.r.t. a norm $\|\cdot\|$ if
$$\forall\theta_1,\theta_2\in\C, \alpha\in[0,1],\Psi(\alpha\theta_1+(1-\alpha)\theta_2)\leq\alpha\Psi(\theta_1)+(1-\alpha)\Psi(\theta_2)-\frac{\Delta\cdot\alpha(1-\alpha)}{2}\|\theta_1-\theta_2\|^2.$$
}
\mypar{Bregman divergence} For any convex function $\Psi:\re^p\to\re$,
the Bregman divergence defined by
$\bregDiv{\Psi}:\re^p\times\re^p\to\re$ is defined as 
$$\bregDiv{\Psi}(\theta_1,\theta_2)=\Psi(\theta_1)-\Psi(\theta_2)-\ip{\grad\Psi(\theta_2)}{\theta_1-\theta_2}.$$
Notice that Bregman divergence is always positive, and convex in the
first argument.

\stocoption{}{
\mypar{$\Delta$-strong convexity w.r.t a function $\Psi$} A function $f:\C\to\re$ is $\Delta$-strongly convex within a set $\C$ w.r.t. a differentiable convex function $\Psi$ if the following holds.
$$\forall\theta_1,\theta_2\in\C,
\alpha\in[0,1],f(\alpha\theta_1+(1-\alpha)\theta_2)\leq\alpha f(\theta_1)+(1-\alpha)f(\theta_2)-\frac{\Delta\cdot\alpha(1-\alpha)}{2}\bregDiv{\Psi}(\theta_1,\theta_2).$$

\mypar{Duality}
The following duality property (Fact \ref{fact:Mink}) of norms will be
useful through the rest of this paper. Recall that for any pair of
dual norms $\|\cdot\|_a$ and $\|\cdot\|_b$, and $x,y\in\re^p$,
Holder's inequality says that $|\langle x,y\rangle|\leq\|x\|_a\cdot\|y\|_b$.
\begin{fact}
	The dual of $\ell_q$ norm is $\ell_{q'}$-norm such that $1/q+1/q'=1$. The dual of $\mink{\cdot}$ is $\minkD{\cdot}$, where for any vector $v\in\re^p$, $\minkD{v}=\max\limits_{w\in\C}\left|\ip{w}{v}\right|.$
\label{fact:Mink}
\end{fact}

\mypar{Gaussian width of a set $\C$} Let $b\sim\mathcal{N}(0,\I_p)$ be a Gaussian random vector in $\re^p$. The Gaussian width of a set $\C$ is defined as $G_{\C}\stackrel{def}{=}{\E_b\left[\sup\limits_{w\in\C}\ip{{b}}{w}\right]}$.}

	\label{fact:GWprop}

\begin{fact}[Concentration of Gaussian width \cite{boucheron2013concentration}]
	Let $W=\sup\limits_{w\in\C}\ip{{b}}{w}$, where $b\sim\mathcal{N}(0,1)^p$ and $\alpha^2=\max\limits_{\theta\in\C}\ltwo{\theta}^2$. Then, $$\Pr\left[\left|W-G_\C\right|\geq u\right]\leq 2e^{-\frac{u^2}{2\alpha^2}}.$$
	\label{fact:GWHighProb}
\end{fact}

\section{Private Mirror Descent and the Geometry of $\C$}
\label{sec:mirrorDesc}

In this section we introduce the well-established  \emph{mirror descent algorithm} \citep{NY83} in the context of private convex optimization. We notice that since mirror descent is designed to closely follow the geometry of the convex set $\C$, we get much tighter bounds than that were known earlier in the literature for a large class of interesting instantiations of the convex set $\C$. More precisely, using private mirror descent one can show that the privacy depends on the Gaussian width (see Section \ref{sec:convNormGauss}) instead of any explicit dependence on the dimensionality $p$. The main technical contribution in the analysis of private (noisy) mirror descent is to express the expected potential drop in terms of the Gaussian width.\stocoption{(See the proof of Theorem \ref{thm:mirrDescUtil} in the full version.)}{(See \eqref{eq:alds132} in the proof of Theorem \ref{thm:mirrDescUtil}.)}
 
\subsection{Private Mirror Descent Method}
\label{sec:privmirrorDescMethod}

In Algorithm \ref{Algo:MirrorDesc} we define our private mirror
descent procedure. The algorithm takes as input a {\em potential
  function} $\Psi$ that is chosen based on the constraint set $\C$. 
$B_\Psi$ refers to the Bregman divergence with respect to $\Psi$. (See Section \ref{sec:convNormGauss}.) If $\empL(\theta;d)$ is not differentiable at $\theta$, we use any sub-gradient at $\theta$ instead of $\grad \empL(\theta;d)$. 
\ktnote{I removed the part about using sampling because it doesn't necessarily add much, and also because the variance of this gradient would show up in the error somewhere.}
\begin{algorithm}[htb]
	\caption{$\A_{\sf Noise-MD}$: Differentially Private Mirror Descent}
	\begin{algorithmic}[1]
		\REQUIRE Data set: $\D=\{d_1,\cdots,d_n\}$, loss function: $\empL(\theta;D)=\frac{1}{n}\sum\limits_{i=1}^n\empL(\theta;d_i)$ (with $\ell_2$-Lipschitz constant $L$ for $\empL$), privacy parameters: $(\epsilon,\delta)$, convex set: $\C$, potential function: $\Psi:\C\to\re$, and learning rate: $\eta:[T+1]\to\re$.
		\STATE  Set noise variance $\sigma^2\leftarrow \frac{32 L^2 T\log^2(T/\delta)}{(\epsilon n)^2}$.
		\STATE Let $\theta_1:$ be an arbitrary point in $\C$.
		\FOR{$t=1$ to $T$}
			{\STATE $ \theta_{t+1}=\arg\min\limits_{\theta\in\C}\ip{\eta_{t+1}\grad(\empL(\theta_t;D)+b_t-\Psi(\theta_t))}{\theta-\theta_t}+\Psi(\theta)$, where $b_t\sim\mathcal{N}(0,\I_p\sigma^2)$.\label{line:MD1}}
		\ENDFOR
		\STATE Output $\privtheta\leftarrow\frac{1}{T}\sum\limits_{t=1}^T\theta_t$.
	\end{algorithmic}
	\label{Algo:MirrorDesc}
\end{algorithm}
\vspace{-0.6cm}
\begin{thm}[Privacy guarantee] Algorithm \ref{Algo:MirrorDesc} is $(\epsilon,\delta)$-differentially private.
\label{thm:mDPrivacy}
\end{thm}
\vspace{-0.2cm}
The proof of this theorem is fairly straightforward and follows from by now standard privacy guarantee of \emph{Gaussian mechanism}~\cite{DKMMN06}, and the strong composition theorem \citep{DRV}. For a detailed proof, we refer the reader to \cite[Theorem 2.1]{BassilyST14}. To establish the utility guarantee in a general form, it will be useful to introduce a symmetric convex body $\Q$ (and the  norm $\minkQ{\cdot}$) w.r.t. which the potential function $\Psi$ is strongly convex. We will instantiate this theorem with various choices of $\Q$ and $\Psi$ depending on $\C$ in Section~\ref{sec:corprivMirrorDesc}. While relatively standard in Mirror Descent algorithms, the reader may find it somewhat counter-intuitive  that $\Q$ enters the algorithm only through the potential function $\Psi$, but plays an important role the analysis and the resulting guarantee. In most of the cases, we will set $\Q=\C$ and the reader may find it convenient to think of that case. \stocoption{}{Our proof of the theorem below closely follows the analysis of mirror descent from \cite{SridharanT10}. 

One can obtain stronger guarantees (typically, $\tilde O(1/(n\epsilon)^2)$) under strong convexity assumptions on the loss function. We defer the details of this result to Appendix \ref{sec:tighterMirrorDesc}.
}

\begin{thm}[Utility guarantee]
Suppose that for any $d \in \D$, the loss function $\empL(\cdot;d)$ is convex and $L$-lipschitz with respect to the $\ell_2$ norm. Let ${\Q}\subseteq\re^p$ be a symmetric convex set with Gaussian width $G_{\Q}$ and $\ell_2$-diameter $\|\Q\|_2$, and let $\Psi:{\C}\to\re$ be $1$-strongly convex w.r.t. $\minkQ{\cdot}$-norm chosen in Algorithm $\A_{\sf Noise-MD}$ (Algorithm \ref{Algo:MirrorDesc}). If  $T=\frac{\ltwo{\Q}^2\epsilon^2 n^2}{L^2\log^2(n/\delta)\left(G_{\Q}^2+\ltwo{\Q}^2\right)}$ and for all $t\in[T+1]$,  $\eta_{t}=\eta=\frac{1}{L\ltwo{{\Q}}\sqrt T}$, then
 	$$\E\left[\empL(\privtheta;D)\right]-\min\limits_{\theta\in\C}\empL(\theta;D)=O\left(\frac{L\sqrt{\left(G^2_\Q+\ltwo{\Q}^2\right)\max\limits_{\theta\in\C}\Psi(\theta)}\log(n/\delta)}{\epsilon n}\right).$$
\label{thm:mirrDescUtil}
\end{thm}

\begin{remark}
Notice that the bound above is scale invariant. For example, given an initial choice of the convex set $\Q$, scaling $\Q$ may reduce $G_\Q$ but at the same time it will scale up the strong convexity parameter. 
\end{remark}
\stocoption{}{
\begin{proof}[Proof of Theorem \ref{thm:mirrDescUtil}]
For the ease of notation we ignore the parameterization of $\empL(\theta;D)$ on the data set $D$ and simply refer to as $\empL(\theta)$. To begin with, from a direct application of Jensen's inequality, we have the following.

\begin{equation}
\empL(\privtheta)-\min\limits_{\theta\in\C}\empL(\theta)\leq\frac{1}{T}\sum\limits_{t=1}^T\empL\left(\theta_t\right)-\min\limits_{\theta\in\C}\empL(\theta)
\label{eq:reg}
\end{equation}

So it suffices to bound the R.H.S. of \eqref{eq:reg} in order to bound the excess empirical risk. In Claim \ref{cl:linearization}, we upper bound the R.H.S. of \eqref{eq:reg} by a sequence of linear approximations of $\empL(\theta)$, thus ``linearizing'' our analysis.
\begin{claim}
Let $\theta^*=\arg\min\limits_{\theta\in\C}\empL(\theta)$. For every $t\in[T]$, let $\gamma_t$ be the sub-gradient of $\empL(\theta_t)$ used in iteration $t$ of Algorithm $\A_{\sf Noise-MD}$ (Algorithm \ref{Algo:MirrorDesc}). Then the convexity of the loss function implies that
$$\frac{1}{T}\sum\limits_{t=1}^T\empL\left(\theta_t\right)-\min\limits_{\theta\in\C}\empL(\theta)\leq\frac{1}{T}\sum\limits_{t=1}^T\ip{\gamma_t}{\theta_t-\theta^*}.$$ 
\label{cl:linearization}
\end{claim}
Thus it suffices to bound $\frac{1}{T}\sum\limits_{t=1}^T\ip{\gamma_t}{\theta_t-\theta^*}$ in order to bound the privacy risk. By simple algebraic manipulation we have the following.  (Recall that $b_t$ is the noise vector used in Algorithm  $\A_{\sf Noise-MD}$.)

\begin{align}
\eta\ip{\gamma_t+b_t}{\theta_t-\theta^*}&=\eta\ip{\gamma_t+b_t}{\theta_t-\theta_{t+1}+\theta_{t+1}-\theta^*}\nonumber\\
&=\underbrace{\eta\ip{\gamma_t+b_t}{\theta_t-\theta_{t+1}}}_{A}+\underbrace{\ip{\eta(\gamma_t+b_t)+\grad\Psi(\theta_{t+1})-\grad\Psi(\theta_t)}{\theta_{t+1}-\theta^*}}_B\nonumber\\
&+\underbrace{\ip{\grad\Psi(\theta_{t})-\grad\Psi(\theta_{t+1})}{\theta_{t+1}-\theta^*}}_C.
\label{eq:gradSplit}
\end{align}
We next upper bound each of the terms $A$, $B$ and $C$ in \eqref{eq:gradSplit}. By Holder's inequality, we write

\begin{align}
A&=\eta\ip{\gamma_t}{\theta_t-\theta_{t+1}} + \eta\ip{b_t}{\theta_t-\theta_{t+1}}\nonumber\\
&\leq\left(\frac{1}{\sqrt 2}\minkQ{\theta_t-\theta_{t+1} }\right)\cdot\left(\eta\sqrt 2\minkQD{\gamma_t}\right)+\left(\frac{1}{\sqrt 2}\minkQ{\theta_t-\theta_{t+1} }\right)\cdot\left(\eta\sqrt 2\minkQD{b_t}\right)\nonumber\\
&\leq \frac{1}{4}\minkQ{\theta_t-\theta_{t+1}}^2+{\eta^2}\minkQD{\gamma_t}^2+\frac{1}{4}\minkQ{\theta_t-\theta_{t+1}}^2+{\eta^2}\minkQD{\b_t}^2\nonumber\\
&=\frac{1}{2}\minkQ{\theta_t-\theta_{t+1}}^2+{\eta^2}\left(\minkQD{\gamma_t}^2+\minkQD{b_t}^2\right)\nonumber\\
\label{eq:alds132}
\end{align}
where we have used the A.M-G.M. inequality in the third step. Taking expectations over the choice of $b_t$, we have 
\begin{align}
\E_{b_t}[A]&\leq\frac{1}{2}\E_{b_t}\left[\minkQ{\theta_t-\theta_{t+1}}^2\right]+{\eta^2}\left(L^2\ltwo{\Q}^2+\E_{b_t}\left[\minkQD{b_t}^2\right]\right).
\label{eq:boundAA}
\end{align}
We now bound $\E_{b_t}\left[\minkQD{b_t}^2\right]$. First notice that $\minkQD{b_t}^2=\sigma^2\left(\max\limits_{\theta\in\Q}\ip{\theta}{v}\right)^2$, where $v\sim\mathcal{N}(0,1)^p$. Let us denote $W=\left(\max\limits_{\theta\in\Q}\ip{\theta}{v}\right)^2$. By Fact \ref{fact:GWHighProb}, we have the following for any $\mu\geq 0$.
\begin{equation}
	\Pr\left[W\geq (\mu+1)^2G^2_\Q\right]\leq 2 e^{-\frac{\mu^2G^2_\Q}{2\ltwo{\Q}^2}}.
\label{eq:abcd134d}
\end{equation}
From \eqref{eq:abcd134d} we have the following.
\begin{align}
\E\left[W\right]&=\int\limits_{0}^\infty \Pr[W\geq x]dx=\int\limits_{0}^{G^2_\Q}\Pr\left[W\geq x\right] dx + \int\limits_{G^2_\Q}^{\infty}\Pr\left[W\geq x\right] dx\nonumber\\
&\leq G^2_\Q+2\int\limits_{G^2_Q}^\infty \exp\left(-\frac{\left(x-G^2_\Q\right)}{2\ltwo{\Q}^2}\right)dx\nonumber\\
&= G^2_\Q+2\int\limits_{0}^\infty \exp\left(-\frac{x}{2\ltwo{\Q}^2}\right)dx=O\left(G^2_\Q+\ltwo{\Q}^2\right).
\label{eq:ahd12} 
\end{align}
Using \eqref{eq:boundAA} and \eqref{eq:ahd12} we have the following:
\begin{align}
\E_{b_t}[A]&\leq\frac{1}{2}\E_{b_t}\left[\minkQ{\theta_t-\theta_{t+1}}^2\right]+{\eta^2}O\left(L^2\ltwo{\Q}^2+\sigma^2\left(G^2_\Q + \ltwo{\Q}^2\right)\right).
\label{eq:boundA}
\end{align}
We next proceed to bound the term $B$ in \eqref{eq:gradSplit}. By the definition of $\theta_{t+1}$, it follows that
\begin{align*}
\ip{\eta(\gamma_t + b_t) - \grad\Psi(\theta_t)}{\theta_{t+1}} + \Psi(\theta_{t+1}) \leq 
\ip{\eta(\gamma_t + b_t) - \grad\Psi(\theta_t)}{\theta^*} + \Psi(\theta^*).
\end{align*}
This implies that 
\begin{align}
B &\leq - \Psi(\theta_{t+1}) + \Psi(\theta^*) + \ip{\grad\Psi(\theta_{t+1})}{\theta_{t+1}-\theta^*} \nonumber\\
&= -B_{\Psi}(\theta_{t+1},\theta^*) \leq 0\label{eq:boundB}.
\end{align}
One can write the term $C$ in \eqref{eq:gradSplit} as follows.
\begin{align}
\bregDiv{\Psi}(\theta^*,\theta_{t})-\bregDiv{\Psi}(\theta^*,\theta_{t+1})-\bregDiv{\Psi}(\theta_{t+1},\theta_t)&=\Psi(\theta^*)-\Psi(\theta_t)-\ip{\grad\Psi(\theta_t)}{\theta^*-\theta_{t}}\nonumber\\
&-\Psi(\theta^*)+\Psi(\theta_{t+1})+\ip{\grad\Psi(\theta_{t+1})}{\theta^*-\theta_{t+1}}\nonumber\\
&-\Psi(\theta_{t+1})+\Psi(\theta_t)+\ip{\grad\Psi(\theta_t)}{\theta_{t+1}-\theta_{t}}=C
\label{eq:boundC}
\end{align}
Notice that since $b_t$ is independent of $\theta_t$,$\E[\ip{b_t}{\theta_t-\theta^*}]=0$. Plugging the bounds \eqref{eq:boundA},\eqref{eq:boundB} and \eqref{eq:boundC} in \eqref{eq:gradSplit}, we have the following. 
\begin{align}
\eta\E\left[\ip{\gamma_t}{\theta_t-\theta^*}\right]&=\eta\E\left[\ip{\gamma_t+b_t}{\theta_t-\theta^*}\right]\nonumber\\
&\leq \bregDiv{\Psi}(\theta^*,\theta_t)-\bregDiv{\Psi}(\theta^*,\theta_{t+1})+{\eta^2}O\left(L^2\ltwo{\Q}^2+\sigma^2\left(G^2_\Q + \ltwo{\Q}^2\right)\right)+\underbrace{\frac{1}{2}\minkQ{\theta_{t}-\theta_{t+1}}^2-\bregDiv{\Psi}(\theta_{t+1},\theta_t)}_D
\label{eq:finalBound1}
\end{align}
In order to bound the term $D$ in \eqref{eq:finalBound1}, we use the assumption that  $\Psi(\theta)$ is $1$-strongly convex with respect to $\|\cdot\|_Q$. This immediately implies that in \eqref{eq:finalBound1} $D\leq 0$. Using this bound, summing over all $T$-rounds, we have\ktnote{This bound on $B_{\Psi}$ needs some justification.}
\begin{align}
\frac{1}{T}\sum\limits_{t=1}^T\E\left[\ip{\gamma_t}{\theta_t-\theta^*}\right]
&\leq\frac{\max\limits_{\theta\in\C}\Psi(\theta)}{\eta T}+\eta O\left(L^2\ltwo{\Q}^2+\sigma^2\left(G^2_\Q + \ltwo{\Q}^2\right)\right)
\label{eq:finalBound2}
\end{align}
In the above we used the following property of Bregman divergence: $B_{\Psi}(\theta^*,\theta_1)\leq \max\limits_{\theta \in \C}\Psi(\theta)$. We can prove this fact as follows. Let $\theta^\dagger=\argmin\limits_{\theta\in\C}\Psi(\theta)$.
By the generalized Pythagorean theorem \cite[Chapter 2]{Rakhlin09}, it follows that $B_{\Psi}(\theta^*,\theta_1)\leq B_{\Psi}(\theta^*,\theta^\dagger)-B_{\Psi}(\theta_1,\theta^\dagger)\leq B_{\Psi}(\theta^*,\theta^\dagger)$. The last inequality follows from the fact that Bregman diverence is always non-negative. Now since $\theta^\dagger$ minimizes $\Psi$ and $\Psi$ is convex, it follows that $\ip{\grad\Psi(\theta^\dagger)}{\theta^*-\theta^\dagger}\geq 0$. This immediately implies $ B_{\Psi}(\theta^*,\theta^\dagger)\leq \Psi(\theta^*)\leq \max\limits_{\theta\in\C}\Psi(\theta)$.

Setting $T=\frac{\ltwo{Q}^2\epsilon^2 n^2}{\log^2(n/\delta)\left(\ltwo{\Q}^2+G^2_\Q\right)}$ and $\eta=\frac{\sqrt{\max\limits_{\theta\in\C}\Psi(\theta)}}{L\ltwo{\Q}\sqrt T}$, and using \eqref{eq:reg} and Claim \ref{cl:linearization}  we get the required bound.
\end{proof}
}

\subsection{Instantiation of Private Mirror Descent to Various Settings of $\C$}
\label{sec:corprivMirrorDesc}

In this section we discuss some of the instantiations of Theorem \ref{thm:mirrDescUtil}.

\mypar{For arbitrary convex set $\C\subseteq\re^p$ with $L_2$-diameter $\ltwo{\C}$} Let $\Psi(\theta)=\frac{1}{2}\ltwo{\theta-\theta_0}^2$ (with some fixed $\theta_0\in\C$) and we choose the convex set $\Q$ to be the unit $\ell_2$-ball in Theorem \ref{thm:mirrDescUtil}. Immediately, we obtain the following as a corollary.

\begin{equation}
\E\left[\empL(\privtheta;D)\right]-\min\limits_{\theta\in\C}\empL(\theta;D)=O\left(\frac{L\sqrt p\ltwo{\C}\log(n/\delta)}{\epsilon n}\right).
\label{eq:abc12}
\end{equation}

This is a slight improvement over~\cite{BassilyST14}.

\mypar{For the convex set $\C\subseteq\re^p$ being a polytope} Let $\C={\sf conv}\{v_1,\cdots,v_k\}$ be the convex hull of vectors $v_i\in\re^p$ such that for all $i\in[p]$, $\ltwo{v_i}\leq \ltwo{\C}$. Fact \ref{fact:Fc1} will be very useful for choosing the correct potential function $\Psi$ in  Algorithm $\A_{\sf Noise-MD}$ (Algorithm \ref{Algo:MirrorDesc}).
\begin{fact}[From \cite{srebro2011universality}]
For the convex set $\C$ defined above, let $\Q$ be the convex hull of $\C$ and $-\C$. The Minkowski norm for any $\theta\in\re^p$ is given by  $\minkQ{\theta}=\inf_{\alpha_1,\cdots,\alpha_k, \sum\limits_{i=1}^k\alpha_i v_i=\theta}\left[\sum\limits_{i=1}^k|\alpha_i|\right]$. Additionally, let $\minkq{\theta}=\inf\limits_{\alpha_1,\cdots,\alpha_k, \sum\limits_{i=1}^k\alpha_i v_i=\theta}\left[\sum\limits_{i=1}^k|\alpha_i|^q\right]^{1/q}$ be a norm for any $q\in(1,2]$. Then the function $\Psi(\theta)=\frac{1}{4(q-1)}\minkq{\theta}^2$ is $1$-strongly convex w.r.t. $\minkq{\cdot}$-norm. 
\label{fact:Fc1}
\end{fact}  

\stocoption{}{
In the following we state the following claim which will be useful later.
\begin{claim}
If $q=\frac{\log k}{\log k-1}$, then the following is true for any $\theta\in\re^p$: $\minkQ{\theta}\leq e\cdot\minkq{\theta}. $
	\label{cl:as2sa}
\end{claim}
 \begin{proof}
 First notice that for any vector $v=\langle v_1,\cdots,v_k\rangle$, $\lone{v}\leq k^{1-1/q}\|v\|_q$. This follows from Holder's inequality. Now setting $q=\log k/(\log k-1)$, we get $\lone{v}\leq e\cdot\|v\|_q$. For any $\theta\in\re^p$, let $a=\langle \alpha_1,\cdots,\alpha_k\rangle$ be the vector of parameters corresponding to $\minkq{\theta}$. From the above, we know that $\lone{a}\leq e\cdot\|a\|_q$. And by definition, we know that $\minkQ{\theta}\leq\lone{a}$. This completes the proof.
\end{proof}
}
\stocoption{One can show that} 
{Claim \ref{cl:as2sa} implies that} if $\Psi(\theta)=\frac{1}{4(q-1)}\minkq{\theta}^2$ and $q=\frac{\log k}{\log k-1}$, then 
$\max\limits_{\theta\in\C}\Psi(\theta)=O(\log k)$. Additionally due to Fact \ref{fact:Fc1},  $\Psi(\theta)$ is $O(1)$-strongly convex w.r.t. $\minkQ{\cdot}$.
\ktnote{We need strong convexity with respect to the $\|\cdot\|_Q$ norm. does that follow from the inequality that I commented out? I think it does and so we should add a statement to that affect. I don't see a immediate proof that $\|\theta\|_{\C} \leq e \|\theta\|_{\c,q}$. Add a sentence.} With the above observations, and observing that $G_\Q=O(\ltwo{\C}\sqrt{\log k})$, setting $\Q$ and $\Psi$ as above, we immediately get the following corollary of Theorem \ref{thm:mirrDescUtil}. Notice that the bound does not have any explicit dependence on the dimensionality of the problem.


\begin{equation}
\E\left[\empL(\privtheta;D)\right]-\min\limits_{\theta\in\C}\empL(\theta;D)=O\left(\frac{L\ltwo{\C}\log k\log(n/\delta)}{\epsilon n}\right).
\label{eq:abc1aaa}
\end{equation}
Notice that this result extends to the standard $p$-dimensional probability simplex: $\C=\{\theta\in\re^p: \sum\limits_{i=1}^p\theta_i=1, \forall i\in[p],\theta_i\geq 0\}$. In this case, the only difference is that the term  $\log k$ gets replaced by $\log p$ in \eqref{eq:abc1aaa}. We remark that applying standard approaches of \cite{JT14,ST15} provides a similar bound only in the case of linear loss functions.

\ktnote{This part need some work. Where in BTN (chapter/section)? What is Q? Also is that $\C_1 \cup \C_2$ or $\C_1 + \C_2$? Also add the chapter/section of BTN for low rank case.}
\mypar{For grouped $\ell_1$-norm} For a vector $x\in\re^p$ and a parameter $k$, the grouped $\ell_1$-norm defined as $\mixed{\theta}=\sum\limits_{i=1}^{\lceil p/k\rceil}\sqrt{\sum\limits_{j=(i-1)k+1}^{\min\{i\cdot k,p\}}\left|\theta_j\right|^2}$. If $\C$ denotes the convex set centered at zero with radius one with respect to $\mixed{\cdot}$-norm, then it follows from union bound on each of the blocks of coordinates in $[p]$ that $\G_{\C}=\sqrt{k\log(p/k)}$. 
\stocoption{
From \cite[Section 5.3.3]{Nem13} it follows that there exists a $\Psi$ which is $1$-strongly convex w.r.t. $\mixed{\cdot}$-norm, and $\max\limits_{\theta\in\C}\Psi(\theta)=O(\sqrt{\log(p/k)})$. Plugging these bounds in Theorem \ref{thm:mirrDescUtil}, we get \eqref{eq:abc1b} as a corollary. For details on $\Psi$ and other interpolation norms, see the full version.
\begin{equation}
\E\left[\empL(\privtheta;D)\right]-\min\limits_{\theta\in\C}\empL(\theta;D)=O\left(\frac{L\sqrt{k\log^2(p/k)}\log(n/\delta)}{\epsilon n}\right).
\label{eq:abc1b}
\end{equation}
}{
In the following we propose the following choices of $\Psi$ depending on the parameter $k$. (These choices are based on \cite[Section 5.3.3]{Nem13}.) For a given $M>1$, divide the coordinates of $\theta$ into $M$ blocks, and denote each block as $\theta^{(j)}$.
$$\Psi(\theta)=\frac{1}{M\xi}\sum_{j=1}^M \left\|\theta^{(j)}\right\|_2^M, M=\begin{cases}
2,& \text{if } \lceil p/k\rceil \leq 2\\
1+1/(\log(p/k)),              & \text{otherwise}
\end{cases},\xi=\begin{cases}
1,& \text{if } \lceil p/k\rceil=1\\
1/2,              & \text{if } \lceil p/k\rceil=2\\
1/(e\log (p/k)) & \text{otherwise}
\end{cases}$$
With this setting of $\Psi(\theta)$ one can show that $\max\limits_{\theta\in\C}\Psi(\theta)=O(\sqrt{\log(p/k)})$. Plugging these bounds in Theorem \ref{thm:mirrDescUtil}, we get \eqref{eq:abc1b} as a corollary. 
\begin{equation}
\E\left[\empL(\privtheta;D)\right]-\min\limits_{\theta\in\C}\empL(\theta;D)=O\left(\frac{L\sqrt{k\log^2(p/k)}\log(n/\delta)}{\epsilon n}\right).
\label{eq:abc1b}
\end{equation}
Similar bounds can be achieved for other forms of interpolation norms, \emph{e.g.,} $L_1,L_2$-interpolation norms:\\  $\|\theta\|_{\alpha,\text{inter}( \ell_1,\ell_2)}=(1-\alpha)\lone{\theta}+\alpha\ltwo{\theta}$ with $\alpha\in[0,1]$. Notice that since the set $\C=\{\theta:\|\theta\|_{\alpha,\text{inter}(\ell_1,\ell_2)\leq 1}\}$ is a subset of  $\C_1+\C_2$, where $\C_1=\{(1-\alpha)\theta:\lone{\theta}\leq 1\}$ and $\C_2=\{\alpha\theta:\ltwo{\theta}\leq 1\}$, it follows that the Gaussian width $G_\C\leq G_{\C_{1}}+G_{\C_{2}}=O((1-\alpha)\sqrt{\log p}+\alpha\sqrt p)$. Additionally from \cite{srebro2011universality}
it follows that there exists a strongly convex function $\Psi(\theta)$ w.r.t. $\|\cdot\|_{\C}$ such that it is $O(1)$ for $\theta\in\C$. While using Theorem \ref{thm:mirrDescUtil} in both of the above settings, we set the convex set $\Q=\C$.
}

\mypar{For low-rank matrices} It is known that the non-private mirror descent extends immediately to matrices \citep{Nem13}. In the following we show that this is also true for the private mirror descent algorithm in Algorithm \ref{Algo:MirrorDesc} ($\A_{\sf Noise-MD}$). For the matrix setting, we assume $\Mtheta\in\re^{p\times p}$ and the loss function $\empL(\Mtheta;d)$ is $L$-Lipschitz in the Frobenius norm $\frob{\cdot}$. From \cite{DTTZ} it follows that if the noise vector $b$ in Algorithm $\A_{\sf Noise-MD}$ is replaced by a matrix $\bm{b}\in\re^{p\times p}$ with each entry of $\bm{b}$ drawn i.i.d. from $\mathcal{N}(0,\sigma^2)$ (with the standard deviation $\sigma$ being the same as in Algorithm $\A_{\sf Noise-MD}$), then the $(\epsilon,\delta)$-differential privacy guarantee holds. In the following we instantiate Theorem \ref{thm:mirrDescUtil} for the class of ${m\times m}$ real matrices with nuclear norm at most one. Call it the set $\C$. (For a matrix $\Mtheta$, $\nuc{\Mtheta}$ refers to the sum of the singular values of $\Mtheta$.) This class is the convex hull of rank one matrices with unit euclidean norm. \cite[Proposition 3.11]{chandrasekaran2012convex} shows that the Gaussian width of $\C$ is $O(\sqrt m)$. \cite[Section 5.2.3]{Nem13} showed that the function $\Psi(\Mtheta)=\frac{4\sqrt{e\log(2m)}}{2^{q}(1+q)}\sum\limits_{i=1}^m\sigma_{i}^{1+q}(\Mtheta)$ with $q=\frac{1}{2\log(2m)}$ is $1$-strongly convex w.r.t. $\nuc{\cdot}$-norm. Moreover, $\max\limits_{\Mtheta\in\C}\Psi(\Mtheta)=O(\log m)$. Plugging these bounds in Theorem \ref{thm:mirrDescUtil} , we immediately get the following excess empirical risk guarantee.
\begin{equation}
\E\left[\empL(\bm{\privtheta};D)\right]-\min\limits_{\Mtheta\in\C}\empL(\Mtheta;D)=O\left(\frac{L\sqrt{m\log m}\log(n/\delta)}{\epsilon n}\right).
\label{eq:abc1c}
\end{equation}

\stocoption{}{
\subsection{Convergence Rate of Noisy Mirror Descent}
\label{sec:learningNoisyMirror}

In this section we analyze the excess empirical risk guarantees of Algorithm \ref{Algo:MirrorDesc} (Algorithm $\A_{\sf Noise-MD}$) as a purely noisy mirror descent algorithm, and \emph{ignoring} privacy considerations. Let us assume that the oracle that returns the gradient computation is noisy. In particular each of the $b_t$ (in Line \ref{line:MD1} of Algorithm $\A_{\sf Noise-MD}$) is drawn independently from distributions which are mean zero and sub-Gaussian with variance $\Sigma_{p\times p}$, where $\Sigma$ is the covariance matrix. For example, this may be achieved by sampling a small number of $d_i$'s and averaging $\grad\empL(\theta_t;d_i)$ over the sampled values. Using the same proof technique of Theorem \ref{thm:mirrDescUtil}, and the observation that $\E_{b\sim\mathcal{N}(0,\I_p)}\left[\max\limits_{\theta\in\C}\left|\ip{\sqrt \Sigma\cdot b}{\theta}\right|\right]=O(\sqrt{\lambda_{\sf max}(\Sigma)}G_\C)$, we obtain the following corollary of Theorem \ref{thm:mirrDescUtil}. Here $\lambda_{\sf max}$ corresponds to the maximum eigenvalue and we set the convex set $\Q=\C$ in Theorem \ref{thm:mirrDescUtil} for the ease of exposition.

\begin{cor}[Noisy  mirror descent guarantee]  Let ${\C}\subseteq\re^p$ be a symmetric convex set with its $\ell_2$ diameter and Gaussian width represented by $\ltwo{\C}$ and $G_{\C}$ respectively, and let $\Psi:{\C}\to\re$ be an $1$-strongly convex function w.r.t. $\mink{\cdot}$-norm chosen in Algorithm $\A_{\sf Noise-MD}$ (Algorithm \ref{Algo:MirrorDesc}). For any $d \in \D$, suppose that the loss function $\empL(\theta;d)$ is convex and $L$-Lipschitz with respect to the $\ell_2$ norm. If for all $t\in[T+1]$,  $\eta_{t}=\eta=\frac{\sqrt{\max\limits_{\theta\in\C}\Psi(\theta)}}{\sqrt{T\left(L^2\ltwo{{\C}}^2+\lambda_{\sf max}(\Sigma)\left(G^2_\C+\ltwo{\C}^2\right)\right)}}$, then the following is true.
	$$\E\left[\empL(\theta^{alg};D)\right]-\min\limits_{\theta\in\C}\empL(\theta;D)=O\left(\frac{\sqrt{\max\limits_{\theta\in\C}\Psi(\theta)}\sqrt{T\left(L^2\ltwo{{\C}}^2+\lambda_{\sf max}(\Sigma)\left(G^2_\C+\ltwo{\C}^2\right)\right)}}{\sqrt T}\right).$$
	Here the expectation is over the randomness of the algorithm.
	\label{cor:mirrDescUtilNoisy}
\end{cor}
Corollary \ref{cor:mirrDescUtilNoisy} above improves on the bound obtained  in the noisy gradient descent literature \cite[Theorem 2]{shamir2013stochastic} as long as the noise follows the mean zero sub-Gaussian distribution mentioned above and the potential function $\Psi$ exists. In particular it improves on the dependence on dimension by removing any explicit dependence on $p$. For different settings of $\Psi$ depending on the convex set $\C$, see Section \ref{sec:corprivMirrorDesc}.
}

\section{Objective Perturbation for Smooth Functions}
\label{sec:smoothStrong}

In this section we show that if the loss function $\empL$ is twice continuously differentiable, then one can recover similar bounds as in Section \ref{sec:mirrorDesc} using the \emph{objective perturbation} algorithm of \cite{CMS11,KST12}. The main contribution in this section is a tighter analysis of objective perturbation using Gaussian width. In the following (Algorithm \ref{Algo:GenObjPert}) we first revisit the objective perturbation algorithm. The $(\epsilon,\delta)$-differential privacy guarantee follows from \cite{KST12}. Theorem \ref{thm:utilObjPert} shows privacy risk bounds that are similar to that in Section \ref{sec:mirrorDesc}.

\begin{remark} The smoothness property of the loss function $\empL$ is used in the privacy analysis. It can be shown that this is in some sense necessary. (See \cite{KST12} for a more detailed discussion.) Standard approaches towards smoothing (like convolving with a smooth function) adversely affects the utility guarantee and results in sub-optimal dependence on the number of data samples ($n$). (See \cite[Appendix E]{BassilyST14}.)
\end{remark}

\stocoption{
\begin{remark}
	We can obtain a tighter error guarantee when the loss function $\empL$ has strong convexity properties. We defer the details to the full version of the paper.
\end{remark}
	}{	
\begin{remark} We define the set $\Q$ as in Theorem \ref{thm:utilObjPert} because we want to both symmetrize and extend the convex set $\C$ to a full-dimensional space. For example, think of the probability simplex in $p$-dimensions as the set $\C$, and $\Q$ to be the $\ell_1$-ball. Also when there exists a differentiable convex function $\Psi:\C\to\re$ such that $\Psi$ is $1$-strongly convex w.r.t. $\minkQ{\cdot}$ and the guarantee in Theorem \ref{thm:utilStrongMirrorDesc} holds w.r.t. $\Psi$, then  Theorem \ref{thm:utilObjPert} is a special case of Theorem \ref{thm:utilStrongMirrorDesc}. This in particular captures the following cases: i) $\Psi(\theta)=\frac{1}{2}\ltwo{\theta}^2$ (and correspondingly $\Q$ being the $\ell_2$-ball), and ii)  $\Psi(\theta)=\sum\limits_{i=1}^p\theta_i\log\theta_i$ (and correspondingly $\Q$ being the $\ell_1$-ball).
\end{remark}
}

\stocoption
{\begin{algorithm}[htb]
		\caption{Objective Perturbation}
		\begin{algorithmic}[1]
			\REQUIRE  Data set: $\D=\{d_1,\cdots,d_n\}$, loss function: $\empL(\theta;D)=\frac{1}{n}\sum\limits_{i=1}^n\empL(\theta;d_i)$ (with $\ell_2$-Lipschitz constant $L$ for $\empL$), privacy parameters: $(\epsilon,\delta)$, convex set: $\C$ (denote the diameter in $\ell_2$-norm by $\|\C\|_2$), upper and lower bounds $\lambda_{max}, \lambda_{min}$ on the eigenvalues of $\grad^2\empL(\theta;d)$ (for all $d$ and for all $\theta\in\F$).
			\STATE Output  $\privtheta\leftarrow\arg\min\limits_{\theta\in\C}\empL(\theta;\D)+\frac{\lambda_{max}}{n\epsilon}\ltwo{\theta}^2+\ip{b}{\theta}$, where $b\sim\mathcal{N}\left(0,~\frac{L^2(2\log(1/\delta))}{(n\epsilon)^2}I_{p\times
				p}\right)$. \label{line:GOPM7}
		\end{algorithmic}
		\label{Algo:GenObjPert}
	\end{algorithm}
	}
{

\begin{algorithm}[htb]
	\caption{Objective Perturbation \citep{KST12}}
	\begin{algorithmic}[1]
		\REQUIRE  Data set: $\D=\{d_1,\cdots,d_n\}$, loss function: $\empL(\theta;D)=\frac{1}{n}\sum\limits_{i=1}^n\empL(\theta;d_i)$ (with $\ell_2$-Lipschitz constant $L$ for $\empL$), privacy parameters: $(\epsilon,\delta)$, convex set: $\C$ (denote the diameter in $\ell_2$-norm by $\|\C\|_2$), upper and lower bounds $\lambda_{max}, \lambda_{min}$ on the eigenvalues of $\grad^2\empL(\theta;d)$ (for all $d$ and for all $\theta\in\F$).
		{\STATE Set $\zeta=\max\left\{\frac{2\lambda_{max}}{n\epsilon}-\min\limits_{\theta\in\C,d\in\D}\lambda_{min}(\grad^2\empL(\theta;d)),0\right\} $.
			\label{line:GOPM2}}
		\STATE Output  $\privtheta\leftarrow\arg\min\limits_{\theta\in\C}\empL(\theta;\D)+\frac{\zeta}{2}\ltwo{\theta-\theta_0}^2+\ip{b}{\theta}$, where $b\sim\mathcal{N}\left(0,~\frac{L^2(2\log(1/\delta))}{(n\epsilon)^2}\I_{p\times
			p}\right)$ and $\theta_0\in\C$ is fixed. \label{line:GOPM7}
	\end{algorithmic}
	\label{Algo:GenObjPert}
\end{algorithm}
	}
\stocoption{
\begin{thm}[Utility guarantee]
	Suppose that $\C \subseteq \re^p$ has diameter $\ltwo{\C}$ and Gaussian width $G_{\C}$. Further suppose that for all $d \in \D$, the loss function $\empL(\cdot;d)$ is twice continuously differentiable, $L$-Lipschitz in the $\ell_2$-norm, and for all $\theta \in \C$, $\|\grad^2 \l(\theta;d)\|$ has spectral norm at most $\lambda_{max}$. Then Algorithm~\ref{Algo:GenObjPert} satisfies the following guarantee.
		$$\E\left[\empL(\privtheta;D)\right]-\min\limits_{\theta\in\C}\empL(\theta;D)=O\left(\frac{LG_\C\sqrt{\log(1/\delta)}+\lambda_{max}\ltwo{\C}^2}{\epsilon n}\right).$$
	\label{thm:utilObjPert}
\end{thm}	
	}{

\begin{thm}[Utility guarantee]
Suppose that $\C \subseteq \re^p$ has diameter $\ltwo{\C}$ and Gaussian width $G_{\C}$. Further suppose that for all $d \in \D$, the loss function $\empL(\cdot;d)$ is twice continuously differentiable, and for all $\theta \in \C$, $\|\grad^2 \l(\theta;d)\|$ has spectral norm at most $\lambda_{max}$. Then Algorithm~\ref{Algo:GenObjPert} satisfies the following guarantees.
\begin{enumerate}
	\item \mypar{Lipschitz case} Suppose that for any $d \in \D$, the loss function $\empL(\cdot;d)$ is convex and $L$-lipschitz w.r.t. the $\ell_2$ norm. Then
 $$\E\left[\empL(\privtheta;D)\right]-\min\limits_{\theta\in\C}\empL(\theta;D)=O\left(\frac{LG_\C\sqrt{\log(1/\delta)}+\lambda_{max}\ltwo{\C}^2}{\epsilon n}\right).$$
	\item \mypar{Lipschitz and strongly convex case} Suppose that for any $d \in \D$, the loss function $\empL(\cdot;d)$ is $L$-lipschitz in the $\ell_2$ norm, and $\Delta$-strongly convex with respect to $\minkQ{\cdot}$, where $\Q$ is the symmetric convex hull of $\C$. If $\Delta\geq\frac{2\ltwo{\C}^2\lambda_{max}}{n\epsilon}$, then the following is true.
	$$\E\left[\empL(\privtheta;D)\right]-\min\limits_{\theta\in\C}\empL(\theta;D)=O\left(\frac{(L G_\C)^2\log(1/\delta)}{\Delta(n\epsilon)^2}\right).$$
\end{enumerate}
\label{thm:utilObjPert}
\end{thm}
}
\stocoption{}{
\begin{proof}
For the ease of notation, we will drop the dependence on the data set $D$, and represent the loss functinon $\empL(\theta;D)$ as $\empL(\theta)$. Let $J(\theta)=\empL(\theta)+\frac{\Delta}{2}\ltwo{\theta}^2$ and let $\Jpriv(\theta)=J(\theta)+\ip{b}{\theta}$. Also let $\nptheta=\arg\min\limits_{\theta\in\C}J(\theta)$. We denote the variance of the noise in Algorithm \ref{Algo:GenObjPert}, by $\sigma^2=\frac{L^2(2\log(1/\delta))}{(n\epsilon)^2}$.

\mypar{Case 1 (Loss function $\empL$ is Lipschitz)}  By the optimality of $\privtheta$, the following is true.
\begin{align}
\Jpriv(\nptheta)&\geq\Jpriv(\privtheta)\nonumber\\
\Leftrightarrow J(\nptheta)+\ip{b}{\nptheta}&\geq J(\privtheta)+\ip{b}{\privtheta}\nonumber\\
\Leftrightarrow J(\privtheta)-J(\nptheta)&\leq \ip{b}{\nptheta-\privtheta}\nonumber\\
\Rightarrow \E\left[ J(\privtheta)-J(\nptheta)\right]&=O\left(\frac{LG_\C\sqrt{\log(1/\delta)}}{\epsilon n}\right).
\label{eq:objLip1} 
\end{align} 
The last equality follows from the definition of Gaussian width and the variance of the noise vector $b$. Let $\theta^*=\arg\min\limits_{\theta\in\C}\empL(\theta)$. From \eqref{eq:objLip1},the definition of $J(\theta)$, and that $\nptheta$ minimizes $J(\theta)$, the following is true.
\begin{align}
\E\left[\empL(\privtheta)-\empL(\theta^*)\right]&=\E\left[J(\privtheta)-J(\theta^*)\right]+\frac{\zeta}{2}\ltwo{\theta^*-\theta_0}^2-\frac{\zeta}{2}\ltwo{\privtheta-\theta_0}^2\nonumber\\
&\leq\E\left[J(\privtheta)-J(\nptheta)\right]+\frac{\zeta}{2}\ltwo{\theta^*-\theta_0}^2\nonumber\\
&=O\left(\frac{LG_\C\sqrt{\log(1/\delta)}+\lambda_{max}\ltwo{\C}^2}{\epsilon n}\right).
\label{eq:objLip2}
\end{align}

\mypar{Case 2 (Loss function $\empL$ is Lipschitz and strongly convex)} First notice that by the definition of Minkowski norm, for any vector $v\in\C$, $\minkQ{v}\geq\ltwo{v}/\ltwo{\C}$. This implies that if $\empL$ is $\Delta$-strongly convex w.r.t. $\minkQ{\cdot}$-norm, then it is $\Delta/\ltwo{\C}^2$ strongly convex w.r.t. $\ltwo{\cdot}$-norm. Hence with the lower bound on $\Delta$-satisfied, $\zeta$ in Algorithm \ref{Algo:GenObjPert} is always zero.

By the definition of strong convexity of $\empL$, the following is true.
\begin{align}
\empL(\theta^*)&\geq\empL(\privtheta)+\frac{\Delta}{2}\minkQ{\privtheta-\theta^*}^2\nonumber\\
\Leftrightarrow \empL(\theta^*)+\ip{b}{\theta^*}-\ip{b}{\theta^*}&\geq\empL(\privtheta)+ \ip{b}{\privtheta}-\ip{b}{\privtheta}+\frac{\Delta}{2}\minkQ{\privtheta-\theta^*}^2\nonumber\\
\Rightarrow \ip{b}{\privtheta-\theta^*}&\geq\frac{\Delta}{2}\minkQ{\privtheta-\theta^*}^2\nonumber\\
\Rightarrow \ip{b}{\frac{\privtheta-\theta^*}{\minkQ{\privtheta-\theta^*}}}&\geq\frac{\Delta}{2}\minkQ{\privtheta-\theta^*}\nonumber\\
\Rightarrow\max\limits_{v\in\Q} \ip{b}{v}&\geq\frac{\Delta}{2}\minkQ{\privtheta-\theta^*}\nonumber\\
\Rightarrow\minkQ{\privtheta-\theta^*}&\leq \frac{2\max\limits_{v\in\Q} \ip{b}{v}}{\Delta}=\frac{2\minkQD{b}}{\Delta}
\label{eq:strongConvexObj1}
\end{align}
In the above we have used the fact  $\empL(\theta^*)+\ip{b}{\theta^*}\leq \empL(\privtheta)+\ip{b}{\privtheta}$ (due to the optimality condition). Using \eqref{eq:strongConvexObj1} we get the following.
\begin{align*}
\empL(\theta^*)+\ip{b}{\theta^*}&\geq\empL(\privtheta)+\ip{b}{\privtheta}\\
\Rightarrow \empL(\privtheta)-\empL(\theta^*)&\leq \minkQD{b}\cdot\minkQ{\privtheta-\theta^*}\\
\Rightarrow \empL(\privtheta)-\empL(\theta^*)&\leq \frac{2\minkQD{b}^2}{\Delta}\\
\Rightarrow \E\left[\empL(\privtheta)\right]-\empL(\theta^*)&=O\left(\frac{\sigma^2 G_\Q^2}{\Delta}\right)\\
&=O\left(\frac{(L G_\C)^2\log(1/\delta)}{\Delta(n\epsilon)^2}\right).
\end{align*}
This completes the proof. In the last step we used the fact that $G_\Q=\Theta(G_\C)$.
\end{proof}
}

\section{Private Convex Optimization by Frank-Wolfe algorithm}
\label{sec:frankWolfe}

The algorithms in the previous section work best when the objective
function is Lipschitz with respect to $\ell_2$ norm.  But in many
machine learning tasks, especially those with sparsity constraint, the
objective function is often Lipschitz with respect to $\ell_1$ norm.
For example, in the high-dimensional linear regression setting e.g. the
classical LASSO algorithm\citep{tibshirani96}, we would like to compute
$\argmin\limits_{\theta,\|\theta\|_1\leq s}\frac{1}{n}\|X\theta-y\|_2^2$.
In the usual case of $|x_{ij}|, |y_j|=O(1)$, $\empL(\theta) =
\frac{1}{n}\|X\theta-y\|_2^2$ is $O(1)$-Lipschitz with respect to
$\ell_1$-norm but is $O(p)$-Lipschitz with respect to $\ell_2$-norm.
So applying the private mirror-descent would result in a fairly loose
bound. In this section, we will show that in these cases it is more
effective to use the private version of the classical Frank-Wolfe
algorithm. In particular, we show that for LASSO, such algorithm
achieves the nearly optimal privacy risk of $\widetilde{O}(1/n^{2/3})$.

\subsection{Frank-Wolfe algorithm}

The Frank-Wolfe algorithm~\citep{frankwolfe} can be regarded as a
``greedy'' algorithm which moves towards the optimum solution
in the first order approximation (see Algorithm~\ref{algo:frankwolfe} for the description).
How fast Frank-Wolfe algorithm converges depends $\empL$'s ``curvature'', defined as follows according to~\cite{clarkson10,jaggi2013revisiting}. We remark that a $\beta$-smooth function on $\C$ has curvature constant bounded by $\beta\|C\|^2$.
\begin{defn}[Curvature constant]
For $\empL:\C\to\re$, define ${\Gamma_\empL}$ as below.
$${\Gamma_\empL}:=\sup\limits_{\theta_1,\theta_2,\in\C, 
		\gamma\in(0,1],
		\theta_3=\theta_1+\gamma(\theta_2-\theta_1)
		}\frac{2}{\gamma^2}\left(\empL(\theta_3)-\empL(\theta_1)-\ip{\theta_3-\theta_1}{\grad \empL(\theta_1)}\right).$$
\label{def:curv}
\end{defn}
\stocoption{}{
\begin{remark}
	One can show (\cite{clarkson10,jaggi2013revisiting}) that for any $q,r\geq 1$ such that $q^{-1}+r^{-1}=1$, ${\Gamma_\empL}$ is upper bounded by $\lambda\|\C\|_q^2$, where $\lambda=\max\limits_{\theta\in\C,\|v\|_q=1}\|\grad^2\empL(\theta)\cdot v\|_r$. 
	\label{remark:89}
\end{remark}
}
\begin{remark}\label{rmk:cf}
One useful bound is for the quadratic programming $\empL(\theta) =
\theta X^TX\theta + \ip{b}{\theta}$. In this case, by~\cite{clarkson10},  $\Gamma_\empL\leq \max_{a,b\in X\cdot\C} \|a-b\|_2^2$.  When $\C$ is centrally symmetric, we have the bound $\Gamma_\empL \leq 4\max_{\theta\in \C}\|X\theta\|_2^2$.
\end{remark}

\begin{algorithm}[htb]
\caption{Frank-Wolfe algorithm}\label{algo:frankwolfe}
\begin{algorithmic}[1]
\REQUIRE $\C\subseteq\re^p$, $\empL:\C\to\re$, $\mu$
\STATE Choose an arbitrary $\theta_1$ from $\C$;
\FOR{$t=1$ to $T-1$}
  \STATE Compute $\htheta_t = \argmin_{\theta\in\C} \langle\grad \empL(\theta_{t}),(\theta-\theta_{t})\rangle$\label{line:xi};
  \STATE Set $\theta_{t+1} = \theta_t + \mu (\htheta_t - \theta_t)$;
\ENDFOR
\STATE return $\theta_{T}$. 
\end{algorithmic}
\end{algorithm}

Define $\theta^\ast=\argmin\limits_{\theta\in\C} \empL(\theta)$. The following theorem shows the convergence of Frank-Wolfe
algorithm.
\begin{thm}[\cite{clarkson10,jaggi2013revisiting}]
If we set $\mu = 1/T$, then 
$\empL(\theta_{T})-\empL(\theta^\ast)= O({\Gamma_\empL}/T)\,.$
\end{thm}

While the Frank-Wolfe algorithm does not necessarily provide faster convergence compared to the
gradient-descent based method, it has two major advantages. First, on
Line~\ref{line:xi}, it reduces the problem to solving a minimization
of linear function. When $\C$ is defined by small number of vertices,
e.g. when $\C$ is an $\ell_1$ ball, the minimization can be done by
checking $\langle\grad \empL(\theta_t),x\rangle$ for each vertex $x$ of $\C$.  This
can be done efficiently. Secondly, each step in Frank-Wolfe takes a
convex combination of $\theta_t$ and $\htheta_t$, which is on the boundary
of $\C$. Hence each intermediate solution is always inside $\C$ (sometimes called
\emph{projection free}), and the final outcome $\theta_T$ is the convex
combination of up to $T$ points on the boundary of $\C$ (or vertices
of $\C$ when $\C$ is a polytope). Such outcome might be desired, for
example when $\C$ is a polytope, as it corresponds to a sparse
solution.  Due to these reasons Frank-Wolfe algorithm has found many
applications in machine
learning~\citep{tongzhang,hazan2012projection,clarkson10}. As we shall
see below, these properties are also useful for obtaining low risk
bounds for their private version.

\subsection{Private Frank-Wolfe Algorithm}
\label{sec:privFW} 
\stocoption{
There are different ways to make Algorithm~\ref{algo:frankwolfe}
private, dependent on the geometry of $\C$. Here we focus on the
important case where $\C$ is a polytope, corresponding to the LASSO
problem. In this case, we apply the exponential mechanism~\citep{MT07}
to achieve privacy.}{
We now present a private version of the Frank-Wolfe algorithm. We can
achieve privacy by replacing Line~\ref{line:xi} in
Algorithm~\ref{algo:frankwolfe} with its private version in one of two
ways. In the first variant, we apply exponential
mechanism~\citep{MT07} to guarantee privacy; and in the second
variant, we apply objective perturbation. The first variant works
especially well when $\C$ is a polytope defined by polynomially many
vertices. In this case, we show that the error depends on the
$\ell_1$-Lipschitz constant, which can be much smaller than the
$\ell_2$-Lipschitz constant.  In particular, the private Frank-Wolfe
algorithm is nearly optimal for the important high-dimensional sparse
linear regression (or compressive sensing) problem. The second variant
applies to general convex set $\C$. In this case, we are able to show
that the risk depends on the Gaussian width of $\C$. The details are
in Appendix~\ref{app:frankL2}.}
Algorithm~\ref{Algo:FWPolytope} describes the private version of
Frank-Wolfe algorithm for the polytope case, i.e. when $\C$ is a
convex hull of a finite set $S$ of vertices (or corners). \stocoption{
In this case, the minimization at Line~\ref{line:xi} in Algorithm~\ref{algo:frankwolfe} can always be achieved by one of the vertices of $\C$.}{
In this case, we know that any linear function is minimized at one
point of $S$ per the following basic fact.
\begin{fact}
Let $\C\subseteq\re^p$ be the convex hull of a compact set
$S\subseteq\re^p$. For any vector $v\in\re^p$,
$\arg\min\limits_{\theta\in\C}\ip{\theta}{v} \cap S\neq \emptyset$.
\label{fact:linOpt}
\end{fact}
}
Since $\theta_{t+1}$ can be selected as one of $|S|$ vertices, by
applying the exponential mechanism~\citep{MT07}, we obtain
differentially private algorithm with risk logarithmically dependent on
$|S|$.  When $|S|$ is polynomial in $p$, it leads to an error bound
with $\log p$ dependence.
While the exponential mechanism can
be applied to the general $\C$ as well, its error would depend on
the size of a cover of the boundary of $\C$, which can be exponential
in $p$, leading to an error bound with polynomial dependence on
$p$. Hence for general convex set $\C$, in $\pfwp$
\stocoption{
}
{
(Algorithm
\ref{Algo:FWGenConvex} in Appendix
\ref{app:frankL2})},
we use objective perturbation instead and obtain
an error dependent on the Gaussian width of $\C$. \stocoption{We defer
the details of objective perturbation to the full version.}{}


\begin{algorithm}[htb]
	\caption{$\pfwe$: Differentially Private Frank-Wolfe Algorithm (Polytope Case)}
	\begin{algorithmic}[1]
		\REQUIRE Data set: $\D=\{d_1,\cdots,d_n\}$, loss function: $\empL(\theta;D)=\frac{1}{n}\sum\limits_{i=1}^n\empL(\theta;d_i)$ (with $\ell_1$-Lipschitz constant $L_1$ for $\ell$), privacy parameters: $(\epsilon,\delta)$, convex set: $\C = conv(S)$ with $\lone{\C}$ denoting $\max_{s \in S} \lone{s}$ and $S$ being the set of corners.
		\STATE Choose an arbitrary $\theta_1$ from $\C$;
		\FOR{$t=1$ to $T-1$}
		\STATE $\forall s\in S, \alpha_s\leftarrow\ip{s}{\grad\empL(\theta_t;D)}+{\sf Lap}\left(\frac{L_1\lone{\C}\sqrt{8T\log(1/\delta)}}{n\epsilon}\right)$, where ${\sf Lap}(\lambda)\sim\frac{1}{2\lambda}e^{-|x|/\lambda}$.
		\STATE $\htheta_t\leftarrow \arg\min\limits_{s\in S}\alpha_s$.
		\STATE $\theta_{t+1}\leftarrow(1-\mu)\theta_t+\mu\htheta_t$, where $\mu=\frac{1}{T+2}$.
		\ENDFOR
		\STATE Output $\privtheta=\theta_T$.
	\end{algorithmic}
	\label{Algo:FWPolytope}
\end{algorithm}

\begin{thm}[Privacy guarantee]
Algorithm \ref{Algo:FWPolytope} is $(\epsilon,\delta)$-differentially private.
\label{thm:privacyFW}
\end{thm}

The proof of privacy follows from a straight forward use of exponential mechanism \citep{MT07,BLST} (the noisy maximum version from \cite[Theorem 5]{BLST}) and the strong composition theorem \citep{DRV}. In Theorem \ref{thm:utilityFW} we prove the utility guarantee for the
private Frank-Wolfe algorithm for the convex polytope case. Define ${{{\Gamma_\empL}}}=\max\limits_{D\in\D} C_{\empL}$
over all the possible data sets in $\D$.

\begin{thm}[Utility guarantee]
	Let $L_1$,$S$ and $\lone{\C}$ be defined as in Algorithms \ref{Algo:FWPolytope} (Algorithm $\pfwe$). Let ${{{\Gamma_\empL}}}$ be an upper bound on the curvature constant (defined in Definition \ref{def:curv}) for the loss function $\empL(\cdot;d)$ that holds for all $d\in\D$.
	In Algorithm $\pfwe$, if we set $T=\frac{{{{\Gamma_\empL}}}^{2/3}(n\epsilon)^{2/3}}{(L_1\lone{\C})^{2/3}}$, then 
	$$
	\E\left[\empL(\privtheta;D)\right]-\min\limits_{\theta\in\C}\empL(\theta;D)=O\left(\frac{{{{\Gamma_\empL}}}^{1/3}\left(L_1\lone{\C}\right)^{2/3}\log(n|S|)\sqrt{\log(1/\delta)}}{(n\epsilon)^{2/3}}\right).$$
	Here the expectation is over the randomness of the algorithm.
\label{thm:utilityFW}
\end{thm}
\stocoption{}{
\begin{proof}
For  ease of notation we hide the dependence of $\empL$ on the data set $D$ and represent it simply as $\empL(\theta)$.
In order to prove the utility guarantee we first invoke the utility guarantee of the non-private noisy Frank-Wolfe algorithm from \cite[Theorem 1]{jaggi2013revisiting}.

\begin{thm} [Non-private utility guarantee \citep{jaggi2013revisiting}]
Assume the conditions in Theorem \ref{thm:utilityFW} and let $\beta>0$ be fixed. Recall that $\mu=1/(T+2)$ and let $\phi_1\in\C$. Suppose that $\langle s_1,\cdots,s_T\rangle$ is a sequence of vectors from $\C$, with $\phi_{t+1}=(1-\mu)\phi_t+\mu s_t$ such that for all $t\in [T]$, 
$$\ip{s_t}{\grad\empL(\phi_t)}\leq \min\limits_{s\in \C}\ip{s}{\grad\empL(\phi_t)}+\frac{1}{2}\beta\mu{{{\Gamma_\empL}}}.$$
Then, 
$$\empL(\phi_T)-\min\limits_{\theta\in\C}\empL(\theta)\leq \frac{2{{{\Gamma_\empL}}}}{T+2}\left(1+\beta\right).$$ 
\label{thm:jaggiUtil}
\end{thm}
Since the convex set $\C$ is a polytope with corners in $S$, if $s_t$ in Theorem \ref{thm:jaggiUtil} corresponds to $\htheta_{t}$ in Algorithm $\pfwe$, and $\phi_t$ corresponds to $\theta_t$ in $\pfwe$, then using the tail properties of Laplace distribution and Fact \ref{fact:linOpt} one can show that with probability at least $1-\zeta$, the term $\beta$ in Theorem \ref{thm:jaggiUtil} is at most $O\left(\frac{L_1\lone{\C}\sqrt{8T\log(1/\delta)}\log(|S|T/\zeta)}{\mu n\epsilon{\Gamma_\empL}}\right)$. Plugging in this bound in Theorem \ref{thm:jaggiUtil}, we immediately get that with probability at least $1-\zeta$,
\begin{equation}
\empL(\theta_T)-\min\limits_{\theta\in\C}\empL(\theta)=O\left(\frac{{{{\Gamma_\empL}}}}{T}+\frac{L_1\lone{\C}\sqrt{8T\log(1/\delta)}\log(|S|T/\zeta)}{n\epsilon}\right).
\label{eq:finalFW1}
\end{equation}
From, \eqref{eq:finalFW1} we can conclude the following in expectation.
\begin{equation}
\E\left[\empL(\theta_T)-\min\limits_{\theta\in\C}\empL(\theta)\right]=O\left(\frac{{{{\Gamma_\empL}}}}{T}+\frac{L_1\lone{\C}\sqrt{8T\log(1/\delta)}\log(TL_1\lone{\C}\cdot|S|)}{n\epsilon}\right).
\label{eq:finalFW15}
\end{equation}
Setting $T=\frac{{{{\Gamma_\empL}}}^{2/3}(n\epsilon)^{2/3}}{(L_1\lone{\C})^{2/3}}$ results in the claimed utility guarantee.
\end{proof}
}

\long\def\cut#1{{}}
\subsection{Nearly optimal private LASSO}
\label{sec:optHighReg}
We now apply the private Frank-Wolfe algorithm $\pfwe$ to the
important case of the sparse linear regression (or LASSO) problem. We show that the private Frank-Wolfe
algorithm leads to a nearly tight $\tilde O(\frac{1}{n^{2/3}})$ bound. 

\mypar{Problem definition} Given a data set
 $D=\{(x_1,y_1),\cdots,(x_n,y_n)\}$ of $n$-samples from the domain
 $D=\{(x,y):x\in\re^p, y\in[-1,1],\linfty{x}\leq 1\}$, and the convex
 set $\C=\ell_1^p$.  Define the mean squared loss,
\begin{equation}\empL(\theta;D)= \frac{1}{2n}\sum_{i\in[n]} (\langle x_i , \theta\rangle - y_i)^2\,.\label{eqn:msqloss}\end{equation}
The objective is to compute $\privtheta\in\C$ to minimize
$\empL(\theta;D)$ while preserving privacy with respect to any change
of individual $(x_i,y_i)$ pair. The non-private setting of the above problem is a variant of the least
squares problem with $\ell_1$ regularization, which was started by the
work of LASSO~\citep{tibshirani96,tibshirani1997lasso} and intensively studied in
the past years~\citep{hastie01statisticallearning,donoho2004higher,candes05,donoho2006compressed,candes2007dantzig,bickel2009simultaneous,bayati2012lasso,RWB,Zhang13}.
\cut{One important reason for using $\ell_1$ regularization is to induce
sparse solutions, i.e. $\theta$ with small number of non-zero coordinates. This
is especially interesting for the so called ``high-dimensional''
setting where $p\gg n$. Indeed, via a long line of
work~\citep{donoho2004higher,candes05,donoho2006compressed,Wainwright06sharpthresholds,candes2007dantzig,bickel2009simultaneous},
it has been shown that under suitable condition of $X$, using $\ell_1$
regularization can indeed produce a nearly optimal sparse solution,
providing theoretical support to the empirical success of LASSO.}

Since the  $\ell_1$ ball is the convex hull of $2p$ vertices, we can apply
the private Frank-Wolfe algorithm $\pfwe$. For the above setting,
it is easy to check that the $\ell_1$-Lipschitz constant is bounded by $O(1)$. Further, by applying the bound on quadratic programming Remark~\ref{rmk:cf}, we
have that $C_{\empL} \leq 4\max_{\theta\in\C} \|X\theta\|_2^2=O(1)$ since
$\C$ is the unit $\ell_1$ ball, and $|x_{ij}|\leq 1$. Hence
$\Gamma=O(1)$. Now applying Theorem~\ref{thm:utilityFW}, we have
\begin{cor}
Let $D=\{(x_1,y_1),\cdots,(x_n,y_n)\}$ of $n$ samples from the domain
$\D=\{(x,y):\linfty{x}\leq 1, |y|\leq 1\}$, and the convex set $\C$
equal to the $\ell_1$-ball. The output $\privtheta$ of Algorithm
$\pfwe$ ensures the following.
$$\E[\empL(\privtheta; D)-\min_{\theta\in\C} \empL(\theta;D)]=O\left(\frac{\log(np/\delta)}{(n\epsilon)^{2/3}}\right).$$
\label{cor:linear}
\end{cor} 

\begin{remark}
Compared to the previous work~\cite{KST12,ST13sparse}, the above upper
bound makes no assumption of \emph{restricted strong convexity} or
\emph{mutual incoherence}, which might be too strong for realistic settings \cite{larry}. 
Also our results significantly improve bounds of \cite{JT14}, from $\tilde O(1/n^{1/3})$ to $\tilde O(1/n^{2/3})$, which considered the case of the set $\C$ being the probability simplex and the loss being a generalized linear model.
\end{remark}

In the following, we shall show that to ensure privacy, the error
bound in Corollary~\ref{cor:linear} is nearly optimal in terms of the
dominant factor of $1/n^{2/3}$.
\begin{thm}[Optimality of private Frank-Wolfe]\label{thm:optPrivFW}
Let $\C$ be the $\ell_1$-ball and $\empL$ be the mean squared loss in equation (\ref{eqn:msqloss}). For every sufficiently large $n$, for
every $(\epsilon,\delta)$-differentially private algorithm $\A$, with
$\epsilon\leq 0.1$ and $\delta=o(1/n^2)$, there exists a data set
$D=\{(x_1,y_1),\cdots,(x_n,y_n)\}$ of $n$ samples from the domain
$\D=\{(x,y):\linfty{x}\leq 1, |y|\leq 1\}$ such that
$$\E[\empL(\A(D);D)-\min_{\theta\in\C}\empL(\theta;D)]=\widetilde{\Omega}\left(\frac{1}{n^{2/3}}\right).$$
\end{thm}
\stocoption{We prove the lower bound by following the fingerprinting codes
	argument of~\cite{BUV13} for lowerbounding the error of
	$(\epsilon,\delta)$-differentially private algorithms. 
}{
We prove the lower bound by following the fingerprinting codes
argument of~\cite{BUV13} for lowerbounding the error of
$(\epsilon,\delta)$-differentially private algorithms. Similar
to~\cite{BUV13} and~\cite{DTTZ}, we start with the following lemma which is implicit in~\cite{BUV13}.
The matrix
$X$ in Theorem~\ref{thm:fingerPrintingLowerBound} is the padded Tardos
code used in~\cite[Section 5]{DTTZ}.  For any matrix $X$, denote by
$X_{(i)}$ the matrix obtained by removing the $i$-th row of $X$.  Call
a column of a matrix a \emph{concensus} column if the entries in the
column are either all $1$ or all $-1$. The sign of a concensus column
is simply the concensus value of the column. Write $w=m/\log m$ and
$p=1000m^2$.
\begin{thm}\label{thm:fingerPrintingLowerBound}
[Corollary 16 from \cite{DTTZ}, restated] Let $m$ be a sufficiently
large positive integer. There exists a matrix $X\in\{-1,1\}^{(w+1)\times
p}$ with the following guarantee.  For each $i\in[1,w+1]$, there
are at least $0.999p$ concensus columns $W_i$ in each $X_{(i)}$. In
addition, for algorithm $\A$ on input matrix $X_{(i)}$ where
$i\in[1,w+1]$, if with probability at least $2/3$, $\A(X_{(i)})$
produces a $p$-dimensional sign vector which agrees with at least
$\frac{3}{4}p$ columns in $W_i$, then $\A$ is not $(\varepsilon,\delta)$
differentially private with respect to single row change (to some
other row in $X$).
\end{thm}

Write $\tau=0.001$. Let $k= \tau wp$.  We first form an $k\times p$
matrix $Y$ where the column vectors of $Y$ are mutually orthogonal
$\{1,-1\}$ vectors.  This is possible as $k\gg p$. Now we construct
$w+1$ databases $D_i$ for $1\leq i\leq w+1$ as follows. For all the
databases, they contain the common set of pair $(z_j,0)$ for $1\leq
j\leq k$ where $z_j = (Y_{j1},\ldots, Y_{jp})$ is the $j$-th row
vector of $Y$. In addition, each $D_i$ contains $w$ pairs $(x_j,1)$
for $x_j=(X_{j1},\ldots, X_{jk})$ for $j\neq i$.  Then
$\empL(\theta;D_i)$ is defined as follows (for the ease of notation in
this proof, we work with the un-normalized loss. This does not affect
the generality of the arguments in any way.)
\[ \empL(\theta;D_i) = \sum_{j\neq i} (x_j \cdot \theta -1)^2 + \sum_{j=1}^{k} (y_j\cdot\theta)^2  = \sum_{j\neq i} (x_j \cdot \theta -1)^2 + k \|\theta\|_2^2\,.\]

The last equality is due to that the columns of $Y$ are mutually
orthogonal $\{-1,1\}$ vectors.  For each $D_i$, consider
$\theta^\ast\in\left\{-\frac{1}{p},\frac{1}{p}\right\}^p$ such that
the sign of the coordinates of $\theta^\ast$ matches the sign for the
consensus columns of $X_{(i)}$.  Plugging
$\theta^\ast$ in $\empL(\theta^\ast;\hat D)$ we have the following,
\begin{align}
\empL(\theta^\ast;\hat D)&\leq \sum\limits_{i=1}^{w}(2\tau)^2+k/p\quad\mbox{since the number of consensus columns is at least $(1-\tau)p$]}\nonumber\\
&=(\tau+4\tau^2)w\,.
\label{eq:p132s}
\end{align} 

We now prove the crucial lemma, which states that if $\theta$ is such
that $\|\theta\|_1\leq 1$ and $\empL(\theta;D_i)$ is small, then
$\theta$ has to agree with the sign of concensus columns of $X_{(i)}$.
\begin{lem}\label{lem:lowerbound}
Suppose that $\|\theta\|_1\leq 1$, and $\empL(\theta;D_i)<1.1\tau
w$. For $j\in W_i$, denote by $s_j$ the sign of the consensus column
$j$. Then we have
\[ |\{j\in W_i\;:\;\sgn(\theta_j) = s_j\}| \geq \frac{3}{4} p\,.\]
\end{lem}
\begin{proof}
For any $S\subseteq\{1,\ldots,p\}$, denote by $\theta|_S$ the
projection of $\theta$ to the coordinate subset $S$. Consider three
subsets $S_1,S_2,S_3$, where
\begin{align*}
S_1 &= \{j\in W_i\;:\; \sgn(\theta_j) = s_j\}\,,\\
S_2 &= \{j\in W_i\;:\; \sgn(\theta_j) \neq s_j\}\,,\\
S_3 &= \{1,\ldots,p\}\setminus W_j\,.
\end{align*}

The proof is by contradiction. Assume that $|S_1|<\frac{3}{4} p$.

Further denote $\theta_i = \theta|_{S_i}$ for $i=1,2,3$. Now we will
bound $\|\theta_1\|_1$ and $\|\theta_3\|_1$ using the inequality
$\|x\|_2 \geq \|x\|_1/\sqrt{d}$ for any $d$-dimensional vector.
\[\|\theta_3\|_2^2 \geq \|\theta_3\|_1^2/|S_3| \geq \|\theta_3\|_1^2/(\tau p)\,.\]
Hence $k\|\theta_3\|_2^2 \geq w\|\theta_3\|_1^2$. But $k\|\theta_3\|_2^2 \leq k\|\theta\|_2^2
\leq 1.1\tau w$, so that $\|\theta_3\|_1 \leq \sqrt{1.1\tau}\leq 0.04$.

Similarly by the assumption of $|S_1|<\frac{3}{4}p$, 
\[\|\theta_1\|_2^2 \geq \|\theta_1\|_1^2/|S_1|\geq 4\|\theta_1\|_1^2/(3p)\,.\]
Again using $k\|\theta\|_2^2<1.1\tau w$, we have that $\|\theta_1\|_1\leq \sqrt{1.1*3/4}\leq 0.91$. 

Now we have $\langle x_i,\theta\rangle-1=\|\theta_1\|_1-\|\theta_2\|_1+\beta_i-1$
where $|\beta_i|\leq \|\theta_3\|_1 \leq 0.04$.  By 
$\|\theta_1\|_1+\|\theta_2\|_1+\|\theta_3\|_1\leq 1$, we have
\[|\langle x_i,\theta\rangle-1| \geq 1-\|\theta_1\|-|\beta_i| \geq 1-0.91-0.04=0.05\,.\]

Hence we have that $\empL(\theta;D_i)\geq (0.05)^2 w \geq 1.1\tau
w$. This leads to a contradiction. Hence we must have
$|S_1|\geq\frac{3}{4} p$.
\end{proof}

With Theorem~\ref{thm:fingerPrintingLowerBound} and
Lemma~\ref{lem:lowerbound}, we can now prove
Theorem~\ref{thm:optPrivFW}.
\begin{proof}
Suppose that $\A$ is private. And for the datasets we constructed
above, 
\[\E[\empL(\A(D_i);D_i)-\min_{\theta}\empL(\theta;D_i)]\leq cw\,,\]
for sufficiently small constant $c$.  By Markov inequality, we have
with probability at least $2/3$,
$\empL(\A(D_i);D_i)-\min_{\theta}\empL(\theta;D_i)\leq 3cw$.  By
(\ref{eq:p132s}), we have $\min\limits_{\theta}\empL(\theta;D_i)\leq
(\tau+4\tau^2)w$.  Hence if we choose $c$ a constant small enough, we
have with probability $2/3$,
\begin{equation}\label{eq:y}
\empL(\A(D_i);D_i)<(\tau+4\tau^2+3c)w \leq 1.1\tau w\,.
\end{equation}
By Lemma~\ref{lem:lowerbound}, (\ref{eq:y}) implies that $\A(D_i)$
agrees with at least $\frac{3}{4}p$ concensus columns in
$X_{(i)}$. However by Theorem~\ref{thm:fingerPrintingLowerBound}, this
violates the privacy of $\A$.  Hence we have that there exists $i$,
such that
\[\E[\empL(\A(D_i);D_i)-\min_{\theta}\empL(\theta;D_i)]>cw\,.\]

Recall that $w=m/\log m$ and $n=w+wp=O(m^3/\log m)$. Hence we have
that 
\[\E[\empL(\A(D_i);D_i)-\min_{\theta}\empL(\theta;D_i)] =
\Omega(n^{1/3}/\log^{2/3} n)\,.\]

The proof is completed by converting the above bound to the normalized
version of $\Omega(1/(n\log n)^{2/3})$.
\end{proof}
}

\bibliographystyle{alpha}
\bibliography{reference}
\stocoption{}{
\appendix
\section{Tighter Guarantees of Mirror Descent for Strongly Convex Functions}
\label{sec:tighterMirrorDesc}

In this section we study Algorithm \ref{Algo:MirrorDesc} (Algorithm $\A_{\sf Noise-MD}$) in the context of strongly convex functions with the following form: Every loss function $\mathcal{L}(\theta;d)$ is $L$-Lipschitz in the $L_2$-norm and $\Delta$-strongly convex with respect to some differentiable convex function $\Psi:\C\to\re$, for any $\theta\in\C$ and $d\in\D$. (See Section \ref{sec:convNormGauss} for a definition.) This setting has previously been studied in~\cite{duchi2010composite,shalev2007logarithmic}. Two common example are: i) $\Psi(\theta)=\frac{1}{2}\ltwo{\theta}^2$ and $\mathcal{L}(\theta;d)$ is $\Delta$-strongly convex w.r.t. $\ltwo{\cdot}$, and ii) for composite loss functions $\mathcal{L}(\theta;d)=g(\theta;d)+\Delta \Psi(\theta)$ if $\Psi(\theta)=\sum\limits_{i=1}^p\theta_i\log(\theta_i)$, then $\mathcal{L}(\theta;d)$ is $\Delta$-strongly convex w.r.t. $\Psi(\theta)$ within the probability simplex which is in turn $1$-strongly convex w.r.t. $\lone{\cdot}$ \cite[Section 5]{duchi2010composite}. In the following we show that one can get a much sharper dependence on $n$ (compared to Theorem \ref{thm:mirrDescUtil}) under strong convexity. 

\begin{remark} \cite{BassilyST14} analyzed the setting of strong convexity w.r.t. $\ell_2$-norm, and in particular provided tight error guarantees. For this case, Theorem~\ref{thm:utilStrongMirrorDesc} leads to similarly tight bounds, and thus the lower bounds in~\cite{BassilyST14} imply that in general, our guarantee cannot be improved.
\end{remark}

\begin{thm}[Utility guarantee for strongly convex functions]
	Let $\Q$ be the symmetric convex hull of $\C$. Assume that every loss function $\mathcal{L}(\theta;d)$ is $L$-Lipschitz in the $\ell_2$-norm and $\Delta$-strongly convex with respect to some differentiable $1$-strongly convex (w.r.t. $\minkQ{\cdot}$) function $\Psi:\C\to\re$, for any $\theta\in\C$ and $d\in\D$. Let $\ltwo{\C}$ be the $\ell_2$-diameter of the set $\C$, and $G_\C$ be the Gaussian width. In Algorithm $\A_{\sf Noise-MD}$ (Algorithm \ref{Algo:MirrorDesc}), if we set $T=\frac{(\ltwo{\C}\cdot n\epsilon)^2}{\left(G_\C^2+\ltwo{\C}^2\right)}$, the potential function to be $\Psi$ and $\eta_t=\frac{2}{\Delta t}$, then following is true.
	$$\E\left[\empL(\privtheta;D)\right]-\min\limits_{\theta\in\C}\empL(\theta;D)=O\left(\frac{L^2 \left(G^2_\C+\ltwo{\C}^2\right)\log (n/\delta)\log(\ltwo{\C}(n\epsilon))}{\Delta (n\epsilon)^2}\right).$$
	Here the expectation is over the randomness of the algorithm.
	\label{thm:utilStrongMirrorDesc}
\end{thm}

\begin{proof}
	For ease of notation, we hide the dependence of $\empL(\theta;D)$ on the data set $D$, and simply represent it as $\empL(\theta)$. The first part of the proof is fairly standard and exactly same as that of Theorem \ref{thm:mirrDescUtil} till \eqref{eq:reg}. Following the same notation, it suffices to bound $\frac{1}{T}\sum\limits_{t=1}^T\empL(\theta_t)-\min\limits_{\theta\in\C}\empL(\theta)$. Rest of the proof differs from Theorem  \ref{thm:mirrDescUtil} to the extent that we now work with a quadratic approximation (Claim \ref{cl:quadrization}) to the loss function instead of a linear application (Claim \ref{cl:linearization}).
	
	\begin{claim}
		Let $\theta^*=\arg\min\limits_{\theta\in\C}\empL(\theta)$. For every $t\in[T]$, let $\gamma_t$ be the sub-gradient of $\empL(\theta_t;D)$ used in iteration $t$ of Algorithm $\A_{\sf Noise-MD}$ (Algorithm \ref{Algo:MirrorDesc}). Then, the following is true.
		$$\frac{1}{T}\sum\limits_{t=1}^T\empL\left(\theta_t;D\right)-\min\limits_{\theta\in\C}\empL(\theta;D)\leq\frac{1}{T}\sum\limits_{t=1}^T\left[\ip{\gamma_t}{\theta_t-\theta^*}-\Delta\cdot\bregDiv{\Psi}(\theta^*,\theta_t)\right].$$ 
		\label{cl:quadrization}
	\end{claim}
	
	The proof of this claim is a direct consequence of the definition of strong convexity. Now using \eqref{eq:finalBound1} from the proof of Theorem \ref{thm:mirrDescUtil} and summing over the $T$ iterations, we have the following. 
	\begin{align}
	&\frac{1}{T}\sum\limits_{t=1}^T\E\left[\ip{\gamma_t}{\theta_t-\theta^*}-\Delta\cdot\bregDiv{\Psi}(\theta^*,\theta_t)\right]\leq \frac{1}{T}\sum\limits_{t=1}^T\E\left[\frac{\bregDiv{\Psi}(\theta^*,\theta_t)-\bregDiv{\Psi}(\theta^*,\theta_{t+1})}{\eta_{t+1}}-\Delta\cdot\bregDiv{\Psi}(\theta^*,\theta_t)\right]\nonumber\\
	&+O\left(\frac{L^2\ltwo{\C}^2+\sigma^2 \left(G_{\C}^2+\ltwo{\C}^2\right)}{T}\right)\sum\limits_{t=1}^T\eta_{t+1}\nonumber\\
	&=\frac{1}{T}\sum\limits_{t=1}^T\E\left[\bregDiv{\Psi}(\theta^*,\theta_t)\left(\frac{1}{\eta_{t+1}}-\frac{1}{\eta_{t}}-\Delta\right)\right]+O\left(\frac{L^2\ltwo{\C}^2+\sigma^2 \left(G_{\C}^2+\ltwo{\C}^2\right)}{T}\right)\sum\limits_{t=1}^T\eta_t
	\label{eq:finalStrongBound1}
	\end{align}
	Now setting $\eta_t=\frac{1}{\Delta\cdot t}$ and using Claim \ref{cl:quadrization} we obtain the following.
	\begin{align}
	\frac{1}{T}\sum\limits_{t=1}^{T}\E\left[\empL(\theta_t)\right]-\empL(\theta^*)&= O\left(\frac{L^2\ltwo{\C}^2+\sigma^2 \left(G_{\C}^2+\ltwo{\C}^2\right)}{\Delta T}\right)\log T\nonumber\\
	&= O\left(\frac{\log T}{\Delta}\left(\frac{L^2\ltwo{\C}^2}{T}+\frac{L^2\left(G_{\C}^2+\ltwo{\C}^2\right)\log(n/\delta)}{(n\epsilon)^2}\right)\right)
	\label{eq:finalStrongBound2}
	\end{align}
	Setting $T=\ltwo{\C}^2(n\epsilon)^2/\left((G^2_\C+\ltwo{\C}^2)\right)$ in \eqref{eq:finalStrongBound2}, we obtain the required excess risk bound as follows:
	$$\E\left[\empL(\privtheta)\right]-\min\limits_{\theta\in\C}\empL(\theta)=O\left(\frac{L^2 \left(G^2_\C+\ltwo{\C}^2\right)\log (n/\delta)\log(\ltwo{\C}(n\epsilon))}{\Delta (n\epsilon)^2}\right).$$
\end{proof}
\section{Missing Details for Private Frank-Wolfe for the $\ell_2$-bounded Case}
\label{app:frankL2}

In this section we provide the details of the private Frank-Wolfe algorithm for the $\ell_2$-bounded case, along with the privacy and utility guarantees.

Here for a data set $\D=\{d_1,\cdots,d_n\}$, define objective function
as the empirical loss function
$\empL(\theta;D)=\frac{1}{n}\sum\limits_{i=1}^n\ell(\theta;d_i)$. We
define $L_2$ the  $\ell_2$-Lipschitz constant,
respectively, of $\empL$ over all the possible data sets.

\begin{algorithm}[htb]
	\caption{$\pfwp$: Differentially Private Frank-Wolfe Algorithm (General Convex Case)}
	\begin{algorithmic}[1]
		\REQUIRE Data set: $\D=\{d_1,\cdots,d_n\}$, loss function: $\empL(\theta;D)=\frac{1}{n}\sum\limits_{i=1}^n\ell(\theta;d_i)$ (with $\ell_2$-Lipschitz constant $L_2$ for $\ell$), privacy parameters: $(\epsilon,\delta)$, convex set: $\C$ bounded in the $\ell_2$-norm, denoted by $\ltwo{\C}$.
		\STATE choose an arbitrary $\theta_1$ from $\C$;
		\FOR{$t=1$ to $T-1$}
		\STATE $\htheta_t=\arg\min\limits_{\theta\in\C}\ip{\grad\empL(\theta_t;D)+b_t}{\theta}$, where $b_t\sim\mathcal{N}(0,\I_p\sigma^2)$ and  $\sigma^2\leftarrow \frac{32 L_2 T\log^2(n/\delta)}{(n\epsilon)^2}$.
		\STATE $\theta_{t+1}\leftarrow(1-\mu)\theta_t+\mu\htheta_t$, where $\mu=\frac{1}{T+2}$.
		\ENDFOR
		\STATE Output $\privtheta=\theta_T$.
	\end{algorithmic}
	\label{Algo:FWGenConvex}
\end{algorithm}

\begin{thm}[Privacy guarantee]
	Algorithm $\pfwp$ (Algorithm \ref{Algo:FWGenConvex}) is $(\epsilon,\delta)$-differentially private.
	\label{thm:privacyFWL2}
\end{thm}

The proof of privacy is exactly same as the proof of privacy in Theorem \ref{thm:mirrDescUtil}. In the following we provide the utility guarantee for Algorithm $\pfwp$.

\begin{thm}[Utility guarantee]
	Let $L_2$, and $\ltwo{\C}$ be defined as in Algorithm \ref{Algo:FWGenConvex} (Algorithm $\pfwp$). Let $G_\C$ the Gaussian width of the convex set $\C\subseteq\re^p$, and let ${{{\Gamma_\empL}}}$ be the curvature constant (defined in Definition \ref{def:curv}) for the loss function $\ell(\theta;d)$ for all $\theta\in\C$ and $d\in\D$.
	In Algorithm $\pfw$ if we set $T=\frac{{{{\Gamma_\empL}}}^{2/3}(n\epsilon)^{2/3}}{(L_2G_{\C})^{2/3}}$, then the excess empirical risk is as follows.
	$$
	\E\left[\empL(\privtheta;D)\right]-\min\limits_{\theta\in\C}\empL(\theta;D)=O\left(\frac{{{\Gamma_\empL}}^{1/3}\left(L_2G_\C\right)^{2/3}\log^2(n/\delta)}{(n\epsilon)^{2/3}}\right).$$		
	Here the expectation is over the randomness of the algorithm  and ${{\Gamma_\empL}}$ is the curvature constant.	   
	\label{thm:utilityFWGenConvex}
\end{thm}

\begin{proof}
 Recall $\sigma^2= \frac{32 L_2 T\log^2(T/\delta)}{(n\epsilon)^2}$. Using the property of Gaussian width (Section \ref{sec:convNormGauss}), and a similar analysis as that of the convex polytope case, we can conclude the following.
	\begin{equation}
	\E\left[\empL(\theta_T)-\min\limits_{\theta\in\C}\empL(\theta)\right]=O\left(\frac{{{{\Gamma_\empL}}}}{T}+\frac{L_2 G_\C \sqrt T\log^2(T/\delta)}{n\epsilon}\right).
	\label{eq:finalFW2}
	\end{equation}
	Setting $T=\frac{{{{\Gamma_\empL}}}^{2/3}(n\epsilon)^{2/3}}{(L_2G_{\C})^{2/3}}$, results in the utility guarantee.
\end{proof}

}
\end{document}